%% file: main.tex
\def\arxivbuild{}

\documentclass{article} %
\usepackage{iclr2027_conference,times}
\usepackage[T1]{fontenc} %

\input{math_commands.tex}

\usepackage{graphicx}
\usepackage{placeins}
\usepackage{amsmath,amssymb,amsthm}
\usepackage{booktabs}
\usepackage{colortbl}
\usepackage{tabularx}
\usepackage{multirow}
\usepackage{xspace}
\usepackage{url}
\usepackage{hyperref}
\hypersetup{hidelinks}
\definecolor{paperrow}{RGB}{238,243,247}

\input{macros}

\newtheorem{theorem}{Theorem}

\newtheorem{corollary}[theorem]{Corollary}

\theoremstyle{definition}

\title{What Should We Freeze?\\[1pt]
{\normalsize Guarded Freezing: Connectivity Shapes the Fine-Tuning of Pretrained Models}}

\ifdefined\arxivbuild
\author{Leonel Aguilar\\
Chair of Cognitive Science, ETH Z\"urich\\
\texttt{aleonel@ethz.ch}}
\iclrfinalcopy
\else
\author{Anonymous authors\\
Paper under double-blind review}
\fi

\begin{document}

\maketitle
\ifdefined\arxivbuild\lhead{Preprint}\fi

\begin{abstract}
When adapting pre-trained models through fine-tuning, freezing weights alone might not preserve performance, as updates elsewhere can change the inputs to the frozen core, ultimately affecting overall performance.
We first analyse the case where a selected frozen core can be isolated and propose \emph{removal-value}, a capacity-based score that approximates HOPE's removal cost averaged over removal orders. We show that in VGG-8, cutting paths from trainable neurons into a frozen core makes selection using this score useful: $70\%$ frozen preserves $5.22\pm0.51$ percentage points more old-task accuracy than DEFT at similar new-task accuracy.
In transformers, shared residual streams leave paths into frozen neurons open.
For this case, we derive \emph{drift-value}, a forward-only proxy for the output disturbance from updating each weight entry under a local update model. In language models, at 40 epochs, this policy exceeds adapted Wanda and RIA freezing scores in settings with substantial retention loss, while its differences from Fisher remain unresolved.
After 160 epochs on Qwen2.5-1.5B, it retains $0.0433\pm0.0102$ more than static Fisher.
In DINOv3 vision-transformer adaptation to point clouds, drift-value retains $0.440$ image accuracy versus $0.187$ for a random
mask of the same count. These results motivate \emph{Guarded Freezing}: select by removal-value when incoming paths are cut, and by drift-value when they remain.
\end{abstract}

\input{sections/intro}
\input{sections/instruments}
\input{sections/theory}
\input{sections/budget}
\input{sections/results}
\input{sections/limitations}
\input{sections/related}
\input{sections/conclusion}
\phantomsection\label{end:maintext}

\subsection*{AI use statement}
We used an AI coding assistant throughout this work. It gave feedback on experiment
protocols and controls. It helped to implement the methods, the evaluation harness,
the experiment orchestration and the provenance checks, and it searched
and summarized literature. The authors provided the ideas, steered and revised every part of the process.
The grouping, monotonicity and scale-invariance results were formalized in Lean~4 with the AI assistant.
The synthetic dataset for the invented-fact task was generated by a seeded script using fixed
sentence templates. Additional AI tools helped in drafting and editing text. The authors read every sentence, created, evaluated and verified every proof and every figure.
The authors take full responsibility for the content of this paper, including text, claims and artifacts produced with the aid of generative AI.

\subsection*{Ethics statement}
This work releases code, result files and calibration arrays, and no
trained weights.
In pre-trained models, selective freezing and pruning can be turned
toward removing an ability, and the pruning control describes what a cut
destroys as much as what it preserves. The adaptation experiments teach
invented facts and one new modality to public models, but no trained
weights are released.

\subsection*{Reproducibility statement}
The supplementary material holds the code, a run registry naming every
experiment with its command line, the result files every number in the
paper is recomputed from, the scripts that build every figure and table,
and the check that recomputes each quoted number from its source.
Appendix~\ref{app:settings} and Appendix~\ref{app:cells} give every
setting. Appendix~\ref{app:proofs} gives the proofs,
and the supplementary material holds Lean~4 formalizations of the grouping,
monotonicity and scale-invariance results of Appendix~\ref{app:proofs}.
Language runs use one 24~GB GPU at 1.7B parameters and an 80~GB GPU at
8B. The numerical consistency checks run on a CPU.

\bibliography{refs}
\bibliographystyle{iclr2027_conference}

\appendix
\renewcommand{\tablename}{Appendix Table}
\renewcommand{\figurename}{Appendix Figure}
\input{appendix/appendix}

\end{document}

%% file: math_commands.tex
\usepackage{amsmath,amsfonts,bm}

\def\eqref#1{equation~\ref{#1}}

\def\1{\bm{1}}

\def\vs{{\bm{s}}}

\DeclareMathAlphabet{\mathsfit}{\encodingdefault}{\sfdefault}{m}{sl}
\SetMathAlphabet{\mathsfit}{bold}{\encodingdefault}{\sfdefault}{bx}{n}

%% file: macros.tex
\newcommand{\nrn}{i}                          %
\newcommand{\layer}{\ell}                     %
\newcommand{\act}{\psi}                       %
\newcommand{\actn}{\act_{\nrn}}               %
\newcommand{\wout}{\mathbf{w}^{\text{out}}}   %
\newcommand{\woutn}{\wout_{\nrn}}
\newcommand{\cpc}{c}                          %
\newcommand{\cpcn}{\cpc_{\nrn}}               %
\newcommand{\cpcdef}{\cpcn=\lVert\woutn\rVert\sqrt{\mathbb{E}_{\mathcal{D}}[\actn^{2}]}}
\newcommand{\Ecap}{E}                         %
\newcommand{\Nn}{N}                           %

\newcommand{\Phin}{\Phi_{\nrn}}
\newcommand{\Phidef}{\Phin=\frac{\Nn\,\cpcn}{\Ecap-\cpcn}}

\newcommand{\coal}{A}                         %

%% file: sections/intro.tex
\section{Introduction}
\label{sec:intro}

Adapting a pre-trained neural network requires a balance: learn a new task while
retaining existing abilities. Selective fine-tuning restricts updates to
a subspace \citep{lora2022} or a subset of parameters \citep{ssu2025,s2ft2024}.
For methods that hold selected weights fixed, the practical question is:
\emph{which parameters should we freeze?}
We study how the answer depends on the connections between trainable
parameters and the computation we want to preserve.

A frozen neuron can receive different inputs after upstream parameters change.
The structural mask of DEFT \citep{hope2026} prevents this by cutting the
connections from trainable neurons into the frozen neurons, while the
reverse connections remain, so new computation can reuse fixed features
(Figure~\ref{fig:schematic}a).
In transformers, frozen neurons usually read a shared residual stream
that other parameters continue to update (Figure~\ref{fig:schematic}b).
The relevant distinction is whether the intervention blocks these paths,
and it gives the two criteria of \emph{Guarded Freezing}.
HOPE's \emph{capacity} measures a unit's contribution and supplies a removal
cost, which also underlies its DEFT freeze rule \citep{hope2026}.
When incoming trainable connections are cut, removal-value selects the
computation worth preserving through that intervention.
When these paths remain, we derive \emph{drift-value}, from one unlabeled
forward pass, to estimate the disturbance caused by parameter updates.

We first analyse the case when isolating a frozen core is possible.
On VGG-8, removal-value selection improves old-task accuracy over DEFT by
$5.2$ percentage points at $70\%$ frozen and similar new-task accuracy, and
removing the cuts removes advantages over random selection.
We then analyse language and vision transformers where the tested interventions leave paths into the frozen core open.
For these cases, we develop and compare drift-value with published selection rules.
We evaluate six language models. At 40 epochs, drift-value leads adapted Wanda and RIA freezing scores where attention freezing leaves substantial retention loss. Its differences from empirical Fisher remain unresolved.
An empirical headroom rule predicts that selection recovers about a third
of the retention loss left by attention freezing. It predicts gains in four
settings with models or rates excluded from the fit.
After 160 epochs, drift-value leads static Fisher by
$0.0433\pm0.0102$ on Qwen2.5-1.5B \citep{qwen25_2024}.
In DINOv3, it raises image retention from $0.097$ to $0.440$ during
point-cloud adaptation under a high-learning-rate stress protocol.

Our contribution is Guarded Freezing: freeze by removal-value when incoming
trainable paths are cut, and by drift-value, computed forward-only, when
they remain.

\begin{figure}[t]
\centering
\includegraphics[width=\textwidth]{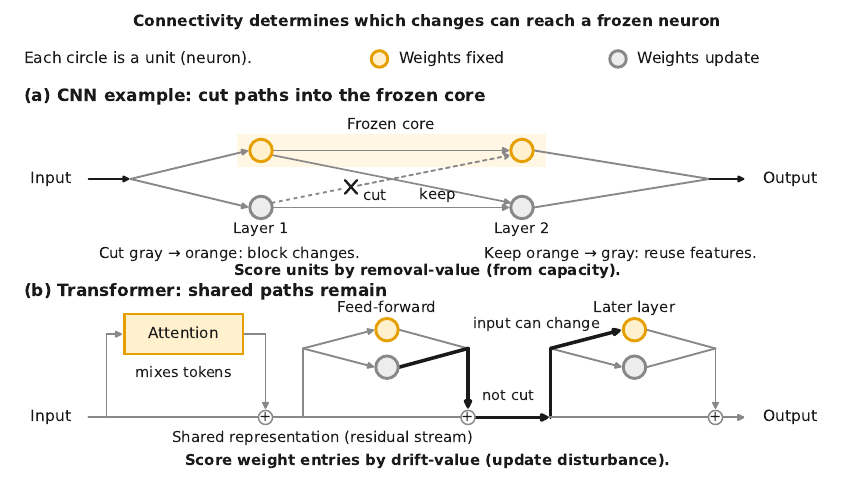}
\caption{\textbf{Connectivity determines which changes can reach a frozen neuron.}
Colors mark fixed or trainable neurons and attention, and arrows show activation flow.
(a) Cut trainable-to-frozen connections to block changes for fixed external
inputs and unchanged normalization. Keep frozen-to-trainable connections
so new computation can reuse fixed features. Our removal-value scores units.
(b) Freezing weights leaves the shared residual stream connected.
The bold path carries changes into a frozen neuron even with attention frozen.
Drift-value scores weight entries by this disturbance.
Plus signs denote addition. Selected layers are shown without normalization.}
\label{fig:schematic}
\end{figure}

%% file: sections/instruments.tex
\section{Connectivity and the freezing problem}
\label{sec:setup}

A \emph{unit} is a hidden feed-forward neuron with input weights and an
output value. An \emph{entry} is one scalar weight.
A \emph{frozen core} is a selected set of units whose weights stay fixed.
The \emph{residual stream} is the running representation that transformer
blocks read and update. A \emph{writer} contributes to it.

\subsection{When is a frozen computation protected?}
\label{sec:which}

If all paths from trainable parameters to a frozen core are removed,
its computation stays fixed for fixed external inputs.
Training then introduces no input drift into the core beyond the initial change caused by cutting its incoming connections.
An \emph{exposed core} keeps paths through which upstream updates can change
its inputs. Freezing its weights alone does not preserve its performance.

Residual connections preserve such paths in the transformer experiments.
HOPE notes that an identity connection in a residual architecture can
bypass its structural mask \citep{hope2026}.
Section~\ref{sec:wiring} tests changing connectivity
within one architecture.

\subsection{What counts as successful protection?}

\emph{Retention} measures performance on the old task after adaptation.
\emph{Acquisition} measures performance on the new task.
We report both: a policy that retains old abilities by failing to learn
the new task does not solve the adaptation problem.

We compare selection rules with published methods and use random masks
to measure the value of selection at the same count.
Frozen units, entries and columns (weights sharing one input feature) impose different update constraints.
Their fractions are not interchangeable parameter budgets.
We report means $\pm$ standard errors (SE), with method differences paired
by seed or transfer scenario. Appendices~\ref{app:settings} and~\ref{app:cells} give the protocols.

%% file: sections/theory.tex
\section{Removal-value and drift-value}
\label{sec:map}

\subsection{Removal-value from capacity}
\label{sec:value}

Following HOPE \citep{hope2026}, we define the capacity of unit $\nrn$ as
\begin{equation}
\cpcdef ,
\label{eq:capacity}
\end{equation}
where $\actn$ is its activation, $\woutn$ its output column, and
$\mathcal D$ the calibration data.
Capacity is the output-column norm multiplied by the activation's root
mean square. We estimate it on calibration data, whereas HOPE uses a
Gaussian surrogate. It measures contribution size, not task loss.
HOPE's removal cost compares this capacity with that of the surviving units.

Averaging the cost over removal orders gives, in the wide-layer
mean-field limit,
\begin{equation}
\Phidef ,
\label{eq:phi}
\end{equation}
where $\Nn$ is the layer width and $\Ecap=\sum_j\cpc_j$.
We call this wide-layer approximation the \emph{removal-value} score.
It approximates the nonterminal order average of the removal cost and
equals HOPE's pruning cost read at the intact layer. DEFT instead reads
cost along one greedy removal path. Replacing surviving capacity by its
mean gives this approximation. Appendix~\ref{app:proofs} defines the
exact average and explains where that replacement can fail.

Within a layer, Equation~\ref{eq:phi} is strictly increasing in capacity,
so it ranks units exactly as capacity does and as the L2 norm of a unit's
Wanda scores \citep{wanda2023} does. This structured score combines the
per-weight scores in a unit's output column using their L2 norm. Our freezing
comparisons instead use Wanda's individual weight scores.
Global thresholding of removal-value gives \emph{removal-value allocation}.
\emph{Uniform allocation} retains the same fraction in every layer.
Appendix~\ref{app:ladder} states what each score reads.

\subsection{Drift-value for parameter updates}
\label{sec:drift}

Let $D$ stack the changes in frozen units' gate pre-activations after
fine-tuning, evaluated on held-out text.
A local model relates the resulting loss to $\lVert D\rVert^2$.
This model assumes a stationary point of the old-task loss, isotropic
curvature in the drift coordinates, and small contributions from the
pathways it omits.
The output perturbation is first order in parameter changes and the
leading loss term at a stationary point is second order. Relating that loss
to retention accuracy requires an empirical fit (Appendix~\ref{app:drift}).

For a linear map $y=Wx$, an update $\delta$ to entry $(i,j)$ changes
$y_i$ by $\delta x_j$.
This perturbation does not depend on the current weight magnitude.
Assume equal-scale updates with independent signs, as an approximation
to normalized optimizer steps.
The expected squared output disturbance is then additive across entries.
Factoring input moments from downstream squared gain $r_i^2$ gives
\begin{equation}
\pi_{ij}=r_i^2\,\mathbb E_{\mathcal D}[x_j^2].
\label{eq:drift}
\end{equation}
Here $r_i^2$ includes the next projection and the intervening activation.
Appendix~\ref{app:drift} gives separate formulas for gate, up, and down projections.
A forward calibration pass supplies the required activation moments.
At a fixed count, we freeze entries with the largest scores.

Our \emph{empirical Fisher} baseline averages squared weight gradients
of the next-token loss, with one gradient per calibration prompt.
It assigns one score to each weight, the Fisher diagonal. We do not
compute the full Fisher matrix of interactions between weights.
For one use of $y=Wx$, the weight gradient is $g_i x_j$, where
$g_i=\partial\mathcal L/\partial y_i$. Fisher therefore depends on
loss gradients, whereas drift-value uses input activity and the forward
gain $r_i^2$ to estimate output disturbance.
A separate factored Fisher control multiplies input and output-gradient
second moments (Appendix~\ref{app:drift}). It is not the main Fisher baseline.

\paragraph{From a score to an adaptation policy.}
Guarded Freezing pairs the score with the intervention: removal-value
when incoming paths are cut, drift-value when they remain.
In transformers we freeze attention, then select feed-forward entries by drift-value.
The measured drift decomposition (Appendix~\ref{app:drift}) identifies
attention as the largest source of input changes, but unlike the entry
score it needs trained checkpoints.

%% file: sections/budget.tex
\section{Isolated frozen core: A controlled connectivity test in a CNN}
\label{sec:wiring}

In the first experiment, we select the core by our removal-value score (Section~\ref{sec:value}) and use DEFT as the selection baseline \citep{hope2026}. Following DEFT's comparison protocol, we adapt a VGG-8 trained on CIFAR-100 \citep{cifar2009} to SVHN digit classification \citep{svhn2011}. Our DEFT port, without HOPE's merge step, reproduces their published comparison (Appendix Table~\ref{tab:hope-reproduction}). We then add additional controls: 1) a random selection core and 2) gentle fine-tuning, an unrestricted full fine-tuning at a lower learning rate. We compare the same selection rules with and without connections from trainable neurons into the frozen core, which separates connectivity from the selection rule.

Retention is accuracy on the original CIFAR-100 classes. Acquisition is accuracy on SVHN digits.
H-score is their harmonic mean, reported on a 0--100 scale.
All policies tested share 20 scenarios (four seeds $\times$ five source draws).
Each draw selects four CIFAR-100 superclasses (20 classes).
The SVHN target always contains all ten digits.
Core policies use learning rate $0.01$, while gentle fine-tuning uses $0.0005$. Core policies therefore train at twenty times its rate. The isolated and exposed versions of each core use the same rate.

\begin{figure}[t]
\centering
\includegraphics[width=\textwidth]{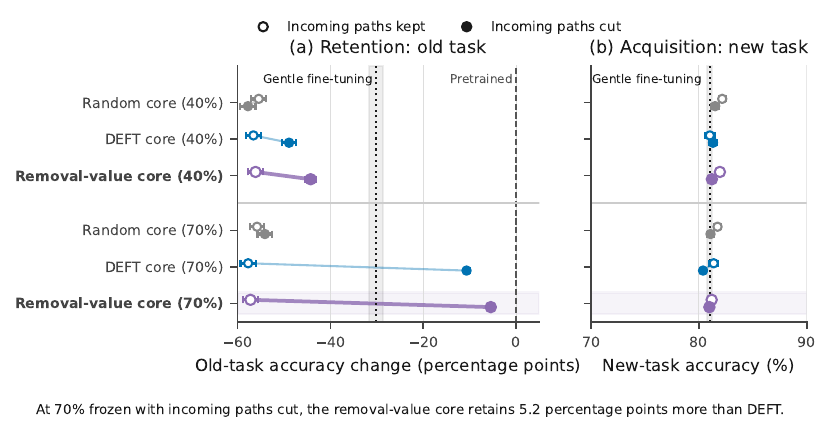}
\caption{\textbf{Cutting incoming trainable paths protects selected cores,
but does not similarly protect random cores.} Twenty VGG-8 adaptations,
with $40\%$ or $70\%$ of units frozen. Left: old-task accuracy change
from pretrained (dashed line). Right: new-task accuracy. Filled/open
markers show incoming paths cut/kept. Core arms train at rate $0.01$.
Gentle fine-tuning (dotted line) uses $0.0005$. Error bars and the
gentle fine-tuning band show one SE. The pale row highlights the
$70\%$ removal-value core. Appendix Table~\ref{tab:isolation} gives H-score.}
\label{fig:regime}
\end{figure}

\paragraph{Removal-value helps when the core is isolated.}
Figure~\ref{fig:regime} compares removal-value, DEFT \citep{hope2026} and random
cores and gentle fine-tuning on the same 20 VGG-8 transfers.
At $40\%$ frozen, isolated removal-value and DEFT cores retain $18.54\%$
and $13.88\%$ old-task accuracy, below gentle fine-tuning's $32.56\%$.
At $70\%$, both exceed it: $57.33\%$ and $52.11\%$.
New-task acquisition remains approximately $81\%$.
The retention advantage over DEFT at $70\%$ is $5.22\pm0.51$
percentage points.

\paragraph{Removing isolation removes the advantage.}
Without isolation, removal-value has no resolved advantage over random
selection at $40\%$ frozen and retains less at $70\%$.
Isolation widens the H-score gap between removal-value and random cores by
$21.35\pm2.07$ points at $40\%$ frozen and $53.83\pm1.75$ at $70\%$.

%% file: sections/results.tex
\section{Selecting parameters in language transformers}
\label{sec:stream}

We now evaluate the score where frozen neurons read a shared residual stream,
the second branch of Guarded Freezing.
The policy freezes all attention parameters, then the feed-forward entries
with the highest drift-value. Other parameters remain trainable, including
block normalization gains.

\subsection{Which entries to freeze}

\paragraph{Tasks and protocol.}
Fine-tuning teaches $150$ generated entity facts (year, founder, place or
count) through four paraphrases each. Acquisition is the fraction whose
answer ranks first against nine same-type distractors under a separate
question, scored by mean token log-likelihood (Appendix~\ref{app:cells}).
Retention is the lowest accuracy across HellaSwag, ARC-Easy and SciQ
\citep{hellaswag2019,arc2018,sciq2017},
called \emph{worst-task accuracy}. Each task is multiple choice, four-way
except for a few ARC-Easy items, with chance accuracy approximately $0.25$ and $200$ evaluation items per seed.
Each run trains for 40 epochs.
The main comparison uses nine seeds on Qwen3-1.7B \citep{qwen3_2025},
bf16 stochastic rounding (bf16 SR, which rounds updates probabilistically),
and learning rate $10^{-4}$. This rate stresses retention.
We call a paired difference resolved when its magnitude exceeds
$\max(0.031,2\,\mathrm{SE})$ (Appendix~\ref{app:headroom}).

\paragraph{What the freezing policy adds.}
On Qwen3-1.7B, full fine-tuning retains $0.263$ and attention freezing alone $0.356$ (Table~\ref{tab:races}). On this model and learning rate, the rounding model estimates that ordinary bf16 discards typical updates
to $35.2\%$ of feed-forward entries. We use that fraction in every feed-forward matrix
as a common budget (Appendix~\ref{app:precision}).
A random mask at this count retains $0.4261$ and drift-value
$0.5022$, a paired gain of $0.0761\pm0.0062$, positive in all nine seeds, at
acquisition $0.9904$.
The gain holds under Kahan-compensated bf16, $0.0739\pm0.0064$,
and at $20\%$ and $50\%$ of every matrix, $0.0856\pm0.0131$ and
$0.0467\pm0.0118$.

\paragraph{Comparison with published selection rules.}
Table~\ref{tab:races} compares drift-value with four rules under the same
attention-frozen protocol. Wanda \citep{wanda2023} and Relative Importance
and Activations (RIA) \citep{ria2024} score weights for pruning. We use
their scores to freeze the highest-scoring weights and update the rest.
Super-Tuning \citep{supertuning2026} also uses Wanda to select trainable
weights. Our Wanda row tests the score in this protocol, not the full
Super-Tuning method.
Source-Shielded Updates (SSU) \citep{ssu2025} freezes whole input columns:
all weights multiplying one input feature. Its budget is rounded to a
whole number of columns. The Fisher row uses the prompt-based empirical
Fisher defined in Section~\ref{sec:drift}. Appendix~\ref{app:cells} gives
the exact scores used.

Drift-value leads Wanda by $0.0439\pm0.0051$ and RIA by $0.0728\pm0.0087$,
positive in all nine seeds, and leads SSU by $0.0150\pm0.0061$ and trails
Fisher by $0.0089\pm0.0068$, both below the decision threshold.
The Fisher comparison remains unresolved on five more models (Figure~\ref{fig:stream}).
Fisher stays inside the threshold on every model, at most $0.0167\pm0.0086$ away
on Qwen2.5-1.5B.
The Wanda and RIA margins follow the loss the attention freeze leaves:
$0.0550\pm0.0072$ and $0.0794\pm0.0119$ on Qwen2.5-1.5B, and inside the threshold
on SmolLM2-1.7B \citep{smollm2_2025} and on Qwen3-0.6B in float32, where the freeze
leaves $0.1189$ of retention to recover or less
(Appendix Table~\ref{tab:headroom}).
The entry scores use 12 short English prompts (268 retained token positions)
for calibration. Fisher requires gradients while drift-value requires only a forward pass.
Across two calibration draws, Jaccard overlap (intersection divided by union)
is $0.6898$ for Fisher and $0.8932$ for drift-value, indicating more stable selection by drift-value. Larger calibration
samples reduce but do not remove this gap.

\begin{table}[t]
\centering
\small
\caption{\textbf{Drift-value leads Wanda and RIA. Its differences from SSU and Fisher remain unresolved.}
Qwen3-1.7B, nine seeds, bf16 SR, learning rate $10^{-4}$, 40 epochs.
Below the divider attention is frozen. Each selection rule additionally freezes $35.2\%$ of
every feed-forward matrix, except the attention-only reference. SSU approximately matches this fraction with whole input columns.
Means $\pm$ SE, differences paired by seed. Rows within each block are
ordered by retention. Random freezes the same count. Pale columns mark the
main readout in every table. Bold numeric differences exceed the stated decision threshold.}
\label{tab:races}
\input{figures/table5_races}
\end{table}

\paragraph{Comparison with adaptation policies.}
Appendix Table~\ref{tab:lora} places drift-value beside a DEFT-selected core \citep{hope2026} and
LoRA \citep{lora2022}.
At matched backbone precision and similar acquisition, drift-value retains
$0.0639\pm0.0099$ more than LoRA on Qwen3-0.6B in float32,
and $0.0533\pm0.0121$ more on Qwen2.5-1.5B in ordinary bf16.
On Qwen2.5-1.5B, a bf16 SR drift-value run trails ordinary-bf16 LoRA by
$0.0472\pm0.0115$. This comparison also changes precision
(Appendix Table~\ref{tab:lora}).

\paragraph{What about removal-value?} With $40\%$ of units frozen, removal-value and random cores have unresolved retention differences (Appendix Table~\ref{tab:unit-policies}). This is consistent with the exposed-core results in Section~\ref{sec:wiring}, where trainable paths remain open.

\begin{figure}[t]
\centering
\includegraphics[width=0.90\textwidth]{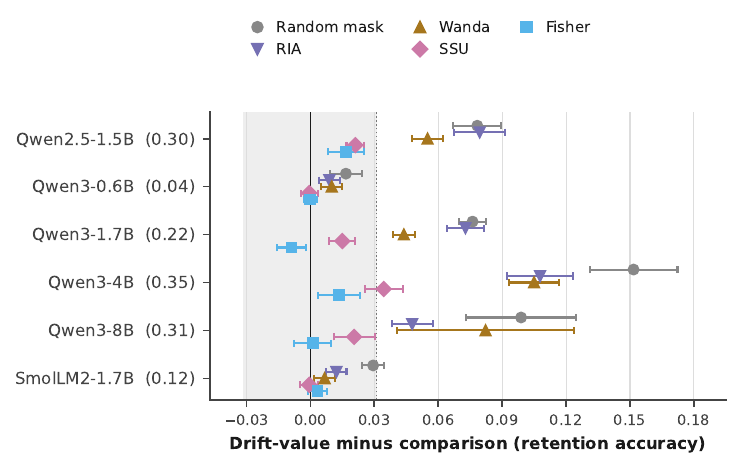}
\caption{\textbf{Drift-value and Fisher do not separate at 40 epochs. Selection gains grow where attention freezing leaves more retention loss.}
Paired retention difference between drift-value and each comparison across
six models, nine seeds. Positive values favor drift-value. Bars show one
paired SE. Parentheses give the retention loss after attention freezing.
Rows group models by family, then by increasing size.
The gray band marks $\pm0.031$, the minimum decision threshold. A resolved
difference must also exceed twice its paired SE.
All policies freeze attention and $35.2\%$ of each feed-forward matrix
(approximately for SSU, which freezes whole input columns). Qwen3-0.6B uses float32. The others use bf16 SR.
Appendix Table~\ref{tab:lora} lists every policy.}
\label{fig:stream}
\end{figure}

\subsection{How long a fixed score remains useful}
\label{sec:shelf}

The main comparisons compute both empirical Fisher and drift-value once
at the starting weights. We call these masks \emph{static}.
Both rankings can change during adaptation,
even while selected weights remain frozen.
On Qwen2.5-1.5B, drift-value minus Fisher is $+0.0022\pm0.0081$ at 40
epochs, $+0.0100\pm0.0062$ at 80 and $+0.0433\pm0.0102$ at 160, positive in
all nine seeds at 160, at acquisition $0.9970$ for both.
On Qwen3-1.7B it stays inside the threshold at every length, $-0.0122$,
$-0.0117$ and $-0.0050$, as it does at 160 epochs,
$+0.0211\pm0.0096$ on SmolLM2-1.7B at $2\times10^{-4}$ and
$-0.0028\pm0.0061$ on Qwen3-0.6B in bf16 SR.
At 160 epochs it also exceeds static SSU by $0.0383\pm0.0074$ on
Qwen3-1.7B and $0.0489\pm0.0063$ on Qwen2.5-1.5B, with similar fact
recall. Drift-value's recomputed set is more stable than
Fisher's: after 40 epochs its overlap with the previous set is $0.7452$
to $0.7614$, versus $0.4709$ to $0.5458$ for Fisher
(Appendix Figure~\ref{fig:shelf}). Refreshing Fisher every 20
epochs brings its retention within the decision threshold of one static
drift-value mask on both models. The long-adaptation
advantage over static Fisher is specific to Qwen2.5-1.5B in these tests.

\FloatBarrier
\section{The same selection rule in a vision transformer}
\label{sec:grow}

Finally, we teach point-cloud recognition to a DINOv3 vision transformer \citep{dinov3_2025} and measure whether it retains its image abilities.
Retention and acquisition are nearest-neighbor accuracy on CIFAR-100 images \citep{cifar2009}
and on ModelNet40 point clouds \citep{modelnet2015}.
All policies use the stress protocol, $30$ epochs at backbone learning rate
$10^{-3}$ (a standard protocol is $10$ epochs at $10^{-4}$), with nine
paired seeds for attention-only, random and drift-value policies at $35.2\%$. Other policies use four seeds.
Pretrained image accuracy is $0.728$.

\paragraph{Selective freezing improves image retention.}
Attention freezing alone retains $0.097$ of image accuracy.
Adding the drift-value entries at $35.2\%$ of every feed-forward matrix
raises retention to $0.440$, against $0.187$ for a random mask of the same
count.
The paired selection gain is $0.253\pm0.008$, positive in all nine seeds.
Point-cloud acquisition is $0.882$ for drift-value and $0.888$ for random
masks.
At the same entry fraction, four-seed SSU, Wanda and RIA freezing
scores retain $0.344$, $0.233$ and $0.212$ image accuracy, with
point-cloud acquisition from $0.883$ to $0.887$. On those four seeds,
drift-value leads Wanda by $0.213\pm0.049$ and RIA by
$0.235\pm0.054$. Its $0.102\pm0.059$ difference from SSU remains
unresolved. These are scores used as freezing rules,
not the original methods' full pipelines.
Full fine-tuning retains $0.025$, a removal-value core $0.036$, and DEFT
$0.044$ with trainable-to-core feed-forward connections cut or $0.030$ with them kept. Residual paths stay open in both.
Appendix Table~\ref{tab:vision-budget} lists every policy at three
budgets and Figure~\ref{fig:vision} the main ones.

\paragraph{Comparison with adapters and expansion.}
The \emph{annex}, our expansion baseline, adds feed-forward units and keeps the original backbone
frozen. Over four seeds it retains $0.725$ of image accuracy and acquires
$0.868$ on point clouds, and a parameter-matched bottleneck adapter
\citep{houlsby2019adapters} retains $0.373$.
The annex retains more than the drift-value mask on a different budget: it
adds units to a frozen backbone, whereas the policy updates existing
parameters, and training the annex's backbone reduces retention at the standard
rate (Appendix~\ref{app:figures}).

\begin{figure}[t]
\centering
\includegraphics[width=0.96\textwidth]{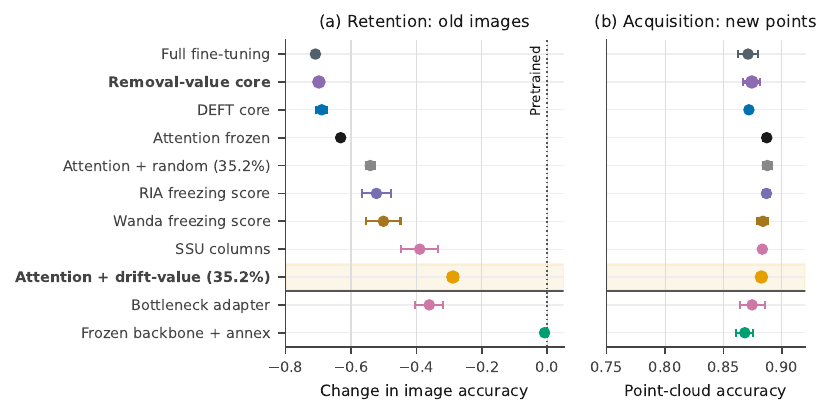}
\caption{\textbf{The annex retains the most image accuracy by adding units.
Drift-value exceeds Wanda and RIA at a matched entry budget.}
DINOv3 point-cloud adaptation under the stress protocol: 30 epochs,
backbone learning rate $10^{-3}$. Left: image accuracy change from
pretrained ($0.728$, dotted line). Right: point-cloud acquisition.
Attention is frozen in every entry-mask policy. At $35.2\%$ of each
feed-forward matrix, drift-value retains $0.253\pm0.008$ more image
accuracy than a matched random mask across nine paired seeds, with
similar acquisition. The random and drift-value
masks and attention freezing use nine seeds. Other rows use four.
Error bars show one SE.
Policies below the divider add parameters. Axis ranges differ.}
\label{fig:vision}
\end{figure}
\FloatBarrier

%% file: figures/table5_races.tex
\setlength{\tabcolsep}{4pt}
\begin{tabular*}{\textwidth}{@{}l@{\extracolsep{\fill}}cc>{\columncolor{paperrow}}c@{}}
\toprule
Policy / selection rule & \begin{tabular}[b]{@{}c@{}}Retention\\(worst-task accuracy)\end{tabular} & \begin{tabular}[b]{@{}c@{}}Acquisition\\(fact recall)\end{tabular} & \begin{tabular}[b]{@{}c@{}}$\Delta$ retention\\vs \textbf{drift-value}\end{tabular} \\
\midrule
Pretrained model & $0.577\pm0.005$ & -- & -- \\
Full fine-tuning & $0.263\pm0.010$ & $0.988\pm0.004$ & {\boldmath $-0.239\pm0.006$} \\
\midrule
Attention frozen & $0.356\pm0.011$ & $0.990\pm0.003$ & {\boldmath $-0.146\pm0.009$} \\
Random mask (control) & $0.426\pm0.008$ & $0.994\pm0.003$ & {\boldmath $-0.076\pm0.006$} \\
RIA freezing score & $0.429\pm0.009$ & $0.990\pm0.003$ & {\boldmath $-0.073\pm0.009$} \\
Wanda freezing score & $0.458\pm0.011$ & $0.987\pm0.003$ & {\boldmath $-0.044\pm0.005$} \\
SSU columns & $0.487\pm0.007$ & $0.988\pm0.004$ & $-0.015\pm0.006$ \\
\textbf{Drift-value} & $0.502\pm0.007$ & $0.990\pm0.004$ & reference \\
Empirical Fisher & $0.511\pm0.006$ & $0.989\pm0.003$ & $+0.009\pm0.007$ \\
\bottomrule
\end{tabular*}
\par\smallskip
{\footnotesize Bold differences exceed $\max(0.031,2\,\mathrm{SE})$ in absolute value.}

%% file: sections/limitations.tex
\section{When does drift-value selection improve retention?}
\label{sec:tools}
\label{sec:cut}

We compare drift-value selection with freezing a random set of the same
size. Two quantities help predict this benefit: score concentration and
the retention loss left by attention-frozen fine-tuning.

\paragraph{Screening by score concentration.}
Under the additive model, selecting fraction $f$ of entries that hold
fraction $q_f$ of the drift-value predicts a gain proportional to $q_f-f$
(Appendix~\ref{app:bound}). A third of feed-forward entries holds
$0.89$--$0.96$ of that value and selection gains $0.0761\pm0.0062$ in the
main setting. Attention-head scores are less concentrated
and the tested head selection has no resolved gain. The controls indicate that much of the benefit comes from
excluding the low-scoring tail.

\paragraph{The empirical headroom rule.}
Let $R_0$ be pretrained retention and $R_A$ retention after fine-tuning
with attention frozen. Define remaining retention loss as $L=R_0-R_A$.
Let $R_f$ and $R_\pi$ denote retention with an additional random or
drift-value-selected mask, each freezing fraction $f$.
Across the evaluated configurations,
\begin{equation}
R_f-R_A\approx fL,\qquad R_\pi-R_f\approx kL.
\label{eq:headroom}
\end{equation}
The first relation estimates the gain from a random mask, the second the
extra gain from drift-value. We fit $k\approx0.32$ on five combinations
of model, learning rate and precision at $f\approx0.352$.
An accuracy loss of ten percentage points after attention freezing
therefore predicts 3.2 points more retention than a random mask.
Across fifteen configurations, gains range from $0.20L$ to $0.46L$ (Appendix~\ref{app:headroom}).

Ten configurations were excluded from the fit. On Qwen3-4B and Qwen3-8B,
the measured gains are $0.43L$ and $0.32L$.
An attention-frozen run therefore gives a rough estimate of whether
selection is worth testing. The predicted gain reaches our minimum
decision threshold of $0.031$ near $L=0.1$.
It is an empirical guide: a resolved gain must still exceed twice its
paired SE (Appendix Figure~\ref{fig:headroom}).

\paragraph{Selection gains across models.}
With attention frozen, we recompute drift-value for each model and
compare it with random selection at the same $35.2\%$ entry budget.
Mean retention is higher with drift-value on all six models over nine
paired seeds (Appendix Table~\ref{tab:transfer-summary}). Gains on Qwen3-4B and Qwen3-8B
are $0.1517\pm0.0206$ and $0.0989\pm0.0259$.

\section{Scope and limitations}
\label{sec:use}

The main CNN test transfers a CIFAR-100-trained VGG-8 to SVHN digits.
The language tests teach $150$ invented facts, recalled closed-book, and
read retention as the worst of three multiple-choice benchmarks. The vision
test grafts point clouds onto one backbone.
Drift-value assumes local, independent, equal-scale updates and simplified downstream gains, and neither the score nor the headroom rule guarantees task accuracy.
At 40 epochs, drift-value and Fisher differences stay inside the decision threshold on six models.
The edge over Fisher at 160 epochs holds on Qwen2.5-1.5B. Three other tested settings remain inside the threshold.
In language models, removal-value, DEFT and random unit cores do not separate at $40\%$ frozen
(Appendix Table~\ref{tab:unit-policies}).
Among the 1.5--1.7B models, only Qwen2.5-1.5B gives a LoRA comparison at the same backbone precision and similar acquisition: on Qwen3-1.7B LoRA acquires $0.1637$ of the facts, and on SmolLM2-1.7B the policy acquires $0.8896$ against LoRA's $0.9859$. SSU, Wanda and RIA enter as scores under one
attention-frozen protocol.
At $35.2\%$ frozen, a seventh model, SmolLM2-360M, retains $0.538$
with drift-value versus $0.529$ with a random mask, with no resolved
gain and slightly lower acquisition.

%% file: sections/related.tex
\section{Related work}
\label{sec:related}

\paragraph{Selective adaptation and continual learning.}
SSU \citep{ssu2025}, AWARe \citep{aware2026}, Super-Tuning
\citep{supertuning2026} and Model-Dowser \citep{modeldowser2026} select
parameters from calibration or self-generated probes.
Model-Dowser bounds an update's output shift by Jacobian sensitivity,
update size (taken as weight magnitude) and input activity, using backward
passes on synthetic samples. MAS \citep{mas2018} scores label-free output
sensitivity by backward passes, and its local variant multiplies ReLU
input and output activations. Drift-value assumes equal-scale updates, so
weight magnitude drops out, and uses forward second moments and a gain
through the next projection. S2FT \citep{s2ft2024} selects heads and
channels, at random or by activation, weight or gradient scores. AWARe
trains attention and freezes feed-forward blocks, the reverse of our
choice, and reports that tuning them costs retention in its setting.
Importance scores for continual learning and sparse fine-tuning use
gradients or trajectories \citep{ewc2017,si2017,fish2021}. Neuron-freezing methods protect
subsets chosen by mean activation \citep{golkar2019}. HOPE \citep{hope2026} supplies capacity, removal cost,
DEFT and isolation. Equation~\ref{eq:phi}'s ratio also appears in the
single-major-player value of an interior oceanic game \citep{milnorshapley1978}.

\paragraph{Pruning scores and allocation.}
Wanda \citep{wanda2023}, FLAP \citep{flap2024}, LLM-Pruner
\citep{llmpruner2023} and SparseGPT \citep{sparsegpt2023} score weights or
structures for pruning. Within a layer, the L2 norm of a unit's Wanda
scores ranks units as capacity does. FLAP's similarity is empirical.
Allocation methods include uniform sparsity \citep{gale2019sparsity}, OWL
\citep{owl2023}, LAMP \citep{lamp2021}, FLAP's global threshold, windowed Shapley
\citep{shapleynonuniform2025}, fixed-shape masks \citep{shearedllama2023},
and whole-block selection \citep{shortgpt2024,sleb2024,gromov2024}.

%% file: sections/conclusion.tex
\section{Conclusion}
\label{sec:discussion}

What to freeze depends on the paths through which adaptation can change
existing computation. When the intervention cuts the paths from trainable
neurons into the frozen core, removal-value selects the computation worth
protecting and improves on DEFT in VGG-8 at similar acquisition.
When the paths remain open, drift-value uses one unlabeled forward
calibration pass. At 40 epochs on six language models, it leads adapted Wanda and RIA scores where attention freezing leaves substantial retention loss, while its differences from empirical Fisher remain unresolved. It exceeds static Fisher on Qwen2.5-1.5B after 160 epochs.
The headroom rule uses the loss after attention freezing to estimate drift-value's retention gain over a random mask of the same size.
We call this policy Guarded Freezing: identify the remaining paths, choose the score that matches the connectivity, and evaluate retention and
acquisition together.

%% file: appendix/appendix.tex
\section{Removal-value: definition and approximation}
\label{app:proofs}

This section explains how removal-value is obtained from HOPE's removal
cost. It separates the exact average over removal orders from the
closed-form approximation used in the experiments.

\subsection{The removal cost}

We use a cooperative-game view to assign importance to individual units.
The units are the players, and a coalition is a set of units that remain
in the layer. Removing one player incurs a cost that depends on which
others remain. Averaging this cost over removal orders gives that unit's
value. Here, ``capacity'' means HOPE's activation-based unit score
(Equation~\ref{eq:capacity}), not the model's parameter count.

Consider a layer of $\Nn\geq2$ units with positive capacities
$\cpc_1,\ldots,\cpc_{\Nn}$ and total $\Ecap=\sum_j\cpc_j$.
After removing a set $S$, let $\coal$ be the surviving units,
$n_\coal=|\coal|$, and $\Ecap_\coal=\sum_{j\in\coal}\cpc_j$.
HOPE's cost of removing unit $\nrn\in\coal$ is
\begin{equation}
\varphi(\nrn\mid S)=\frac{n_\coal\,\cpcn}{\Ecap_\coal-\cpcn}.
\label{eq:marginal}
\end{equation}
We average this cost over uniformly random removal orders, conditional
on the unit not being last in its layer. The final removal is excluded
because its denominator is zero. A minimum surviving width used in
pruning is a separate allocation constraint.

Different removal orders can accumulate different total costs, even
when they remove the same set of units. Thus these costs need not be
marginal changes of a total-cost function defined on sets. We use their
random-order average directly.

\subsection{The closed-form approximation}

The exact nonterminal average is
\begin{equation}
V_{\nrn}=\frac{1}{\Nn-1}\sum_{k=0}^{\Nn-2}(\Nn-k)\,\cpcn\;
      \mathbb{E}_k\!\left[\frac{1}{\Ecap-C_S-\cpcn}\right],
\label{eq:exact}
\end{equation}
where $\mathbb E_k$ averages over uniform $k$-element subsets of the
other units and $C_S=\sum_{j\in S}\cpc_j$.
The mean-field approximation replaces the remaining capacity by its mean,
$(\Ecap-\cpcn)(\Nn-1-k)/(\Nn-1)$. This gives
\[
V_{\nrn}^{\mathrm{MF}}
=\frac{\cpcn}{\Ecap-\cpcn}(\Nn-1+H_{\Nn-1})
=\Phin\left(1+\frac{H_{\Nn-1}-1}{\Nn}\right),
\]
where $H_m=\sum_{j=1}^{m}1/j$ and
$\Phin=\Nn\cpcn/(\Ecap-\cpcn)$ is Equation~\ref{eq:phi}.
The relative finite-width correction within this approximation is
$O(\log\Nn/\Nn)$. This does not bound the error from replacing the
reciprocal capacity by the reciprocal of its mean. That step is least
reliable near the end of an order, when few units remain. Even if no unit carries a large share of total capacity, this error
need not be small.

The direction of the approximation error follows from convexity.
For positive remaining capacity $X$, Jensen's inequality gives
$\mathbb E[1/X]\geq1/\mathbb E[X]$. Applying it at each removal
position yields $V_{\nrn}\geq V_{\nrn}^{\mathrm{MF}}\geq\Phin$.
Thus the closed form is a lower approximation to the order average,
without a general relative-error guarantee. Its monotonicity and scale
invariance are exact.

\subsection{Comparison with DEFT}

DEFT records each unit's cost at the step where one greedy removal path
prunes it (\citealp{hope2026}, Section 11.2.2, Equation 24). Each pruning
step within a layer removes its lowest-capacity active unit. If $n$ units remain before
removal, its cost is $n/(n-1)$ times the removed unit's capacity divided
by the surviving units' mean capacity. The random-order average instead
samples many survivor sets. Removal-value evaluates the same cost at
the intact layer as an approximation to that average.

\subsection{Grouping whole layers leaves the average unchanged}

We next ask whether grouping removals by layer changes the value of a
unit. Let $U_{\layer}$ be the units in layer $\layer$. The removal cost depends
only on removals in that layer:
\begin{equation}
\varphi(\nrn\mid S)=\varphi^{(\layer)}\!\big(\nrn\mid S\cap U_{\layer}\big),
\qquad \nrn\in U_{\layer}.
\tag{Sep}\label{eq:sep}
\end{equation}
A grouped order first orders the groups uniformly, then orders the units
uniformly within each group, keeping each group contiguous.

\begin{theorem}[Invariance under grouping whole layers]
Under \eqref{eq:sep}, suppose each layer lies entirely in one group of a
partition $\{B_1,\ldots,B_m\}$. The grouped random-order average
$\mathrm{Ow}_{\nrn}$, the unrestricted average $\mathrm W_{\nrn}$, and the
average within the unit's layer are equal:
\[
\mathrm{Ow}_{\nrn}\big(\varphi;\{B_1,\ldots,B_m\}\big)
=\mathrm W_{\nrn}(\varphi)
=\mathrm W_{\nrn}\big(\varphi^{(\layer)}\big).
\]
All averages use the same nonterminal-position convention.
\end{theorem}

\begin{proof}
By \eqref{eq:sep}, only the order of units in $U_{\layer}$ affects the
cost. In unrestricted orders this relative order is uniform. In grouped
orders it is also uniform, since all units of $U_{\layer}$ belong to the
same uniformly ordered group. Conditioning on the unit not being last
in its layer preserves the same distribution in both cases.
\end{proof}

\paragraph{Splitting a layer can change the average.}
Consider three units $i,j,k$ in one layer and the partition
$\{\{i\},\{j,k\}\}$. Give $i$ cost one when exactly one other unit
precedes it, and zero otherwise. Conditional on $i$ not being last, its
unrestricted average is $1/2$. Its grouped average is zero, because
$i$ can only be first or last. This counterexample shows why the theorem
requires whole layers. It does not imply that every split changes values.

An empirical check on Qwen3-1.7B groups units carrying at least $1\%$
of layer capacity separately from the rest. If none qualifies, it groups
the largest unit separately. Unlike Equation~\ref{eq:exact}, this check
stops each order with at least $5\%$ of the layer remaining and averages each
unit's cost only over draws in which it is removed. With 4000 draws,
the median absolute relative change from unrestricted ordering is
$0.0011$, versus $0.0002$ between independent unrestricted estimates.
These are medians of within-layer medians. The median within-layer
Spearman correlation with removal-value exceeds $0.999999$.
These findings concern the sampled layers and this stopping rule, not
general invariance under splitting a layer.

\begin{corollary}[Within-layer ranking]
At fixed $\Nn$ and $\Ecap$, removal-value is strictly increasing in
$\cpcn$, so it ranks units exactly as capacity does. This statement is
about the closed form. The finite-width ranking of the exact order
average is a separate question.
\end{corollary}

\begin{proof}
The derivative of $\Nn c/(\Ecap-c)$ is
$\Nn\Ecap/(\Ecap-c)^2>0$ for $0<c<\Ecap$.
\end{proof}

\begin{corollary}[Layer-local information]
Under uniform removal orders, the exact value of a unit depends only on
capacities in its own layer. The closed-form approximation uses only
$\cpcn$, $\Nn$, and $\Ecap$ and is unchanged by rescaling all capacities
in that layer.
\end{corollary}

Neither score models interactions between layers. The exact average can
depend on the full within-layer capacity distribution. The pruning
control in Appendix~\ref{app:pruning-control} therefore tests unit
ranking separately from allocation across layers.

\section{Data required by each score}
\label{app:ladder}

The scores differ in what information they need. This section distinguishes
calibration before training from measurements that require trained models,
so their computational requirements are explicit.

Appendix Table~\ref{tab:data} lists these requirements. Calibration-size
and random-token controls appear in Appendix~\ref{app:additional-controls}.

\begin{table}[!htbp]
\centering
\footnotesize
\caption{\textbf{Removal-value and drift-value need only forward calibration.}
Fisher also needs backward passes. The headroom rule needs a fitted slope
and an attention-frozen run. The shaded column lists required data.}

\label{tab:data}
\begin{tabularx}{\textwidth}{@{}>{\raggedright\arraybackslash}p{0.22\textwidth}>{\columncolor{paperrow}\raggedright\arraybackslash}X>{\raggedright\arraybackslash}X@{}}
\toprule
Quantity & Required data & Interpretation \\
\midrule
Rounding-induced freeze & Weights, learning rate and format & Estimates rounding-limited updates \\
\textbf{Removal-value} & One forward calibration pass & Ranks units by contribution size \\
\textbf{Drift-value} & One forward calibration pass & Estimates local update disturbance \\
\textbf{Vision drift-value} & One pass over the 256 \textbf{removal-value} calibration images & Estimates local update disturbance \\
Remaining retention loss & Pretrained evaluation and an attention-frozen reference run & Input to the fitted headroom rule \\
Drift projection shares & Fine-tuned checkpoints & Measures the contribution of each writer family \\
Empirical Fisher & One backward pass per calibration prompt, next tokens as targets & Ranks weights by mean squared prompt-loss gradient \\
\bottomrule
\end{tabularx}
\end{table}

\FloatBarrier
\section{Pruning and expansion protocols}
\label{app:settings}

Pruning removes units, while an annex adds trainable units to a frozen
backbone. These controls test the removal and expansion sides of the
paper's argument. Their protocols differ from the freezing experiments.

\paragraph{Pruning models and precision.}
The eight dense language models are Qwen3-0.6B, 1.7B, 4B, 8B and 32B,
Mistral-7B v0.1, SmolLM2-360M (the language backbone of SmolVLM-500M),
and OLMo-2-7B (November 2024 release). Qwen3-0.6B and Qwen3-1.7B use
float32. The other six models use bf16. Comparisons within a model use the
same precision. Adaptation precisions are specified separately in
Appendix~\ref{app:cells}.

\paragraph{Pruning calibration and evaluation.}
The default calibration is one forward pass over 16 blocks of 1024 tokens
from GSM8K \citep{gsm8k2021}, retaining all positions. Mistral-7B instead uses 12 short built-in prompts from mixed
domains, retaining all 365 token positions. A calibration-domain control uses approximately equal
token shares from maths, open-domain questions, and encyclopedic prose.
Evaluation draws 500 valid items per task and seed from HellaSwag's
validation split and the ARC-Easy and SciQ test splits. We select the
continuation with the highest mean token log-likelihood. HellaSwag uses
its context and candidate endings. ARC-Easy and SciQ use the question
followed by ``Answer:'' and candidate answers, without supporting text.
Language adaptation uses the same scoring protocol with
200 items per task (Appendix~\ref{app:cells}).

\paragraph{Allocation and repair.}
The pruning control compares uniform retained fractions with allocation
by removal-value. A separate repair experiment allocates widths using
measured loss from removing a whole feed-forward block, evaluated on
200 items per task. Its minimum retained fractions are $2\%$ per layer
and $15\%$ in the first three layers. Repair distils the uncut model's
output distribution into the pruned model on 10 million tokens per
round from the maths and encyclopedic mixture, with sequence length
2048, batch size 4, and initial learning rate $10^{-4}$.

\paragraph{Vision expansion and evaluation.}
The annex adds $10\%$ to each feed-forward layer's width, initialized to
preserve the backbone's output. We measure image retention on CIFAR-100
and point-cloud acquisition on ModelNet40 using $k$-nearest-neighbor
accuracy with $k=20$, 2000 reference embeddings, and 1000 query
embeddings. Appendix~\ref{app:cells} specifies training. The VGG-8
connectivity experiment is a separate protocol (Section~\ref{sec:wiring}).

\section{Transformer experiment configurations}
\label{app:cells}

The freezing experiments ask which parameters can remain fixed while a
model learns a new task. This section defines the language and vision
tasks, evaluation criteria, masks, and common baseline implementations.

\paragraph{Language adaptation.}
We train on 150 of 200 invented facts for 40 epochs, with batch size 8
and learning rate $10^{-4}$ unless stated otherwise. Each fact assigns a
synthetic entity a year, founder, place or count, expressed through four
declarative templates. Evaluation uses a question template and ranks the
correct continuation against nine same-kind distractors by mean token
log-likelihood. Acquisition is accuracy on the 150 trained facts. The
remaining 50 are untrained leakage controls. Language adaptation uses
AdamW with $(\beta_1,\beta_2)=(0.9,0.999)$, $\epsilon=10^{-8}$,
weight decay $0.01$, and a constant learning rate. Training sequences are
truncated to 64 tokens.
Retention is the minimum accuracy over HellaSwag, ARC-Easy and SciQ,
with 200 items per task and chance accuracy approximately $0.25$. ARC-Easy contains a few items with other than four choices. A mean retention
within $0.05$ of chance is classified as near chance. Comparisons use
nine paired seeds unless stated otherwise, with eighteen in the extended
SmolLM2 comparison. A difference is resolved when its magnitude exceeds
$\max(0.031,2\,\mathrm{SE})$, using the paired standard error.
The main models are Qwen2.5-1.5B \citep{qwen25_2024}, Qwen3-0.6B,
1.7B, 4B and 8B \citep{qwen3_2025}, and SmolLM2-1.7B
\citep{smollm2_2025}. Additional controls use SmolLM2-360M.

\paragraph{Learning rates.}
The main rate is $10^{-4}$. Lower-rate controls use $5\times10^{-5}$
on Qwen3-1.7B and $2\times10^{-5}$ on Qwen3-0.6B. A higher-rate
SmolLM2-1.7B control uses $2\times10^{-4}$.

\paragraph{Numerical precision and exact freezing.}
We compare ordinary bf16, bf16 with stochastic rounding, bf16 with
Kahan compensation, and float32. Ordinary bf16 discards sufficiently
small updates when rounding to nearest. Appendix~\ref{app:precision}
estimates the resulting frozen fraction. Stochastic rounding preserves
small updates in expectation \citep{gupta2015,zamirai2020}. Kahan
compensation accumulates rounding residuals in a buffer \citep{zamirai2020}.
Explicit masks zero frozen-entry updates, keep their compensation
buffers at zero where applicable, and use snapshot restoration for
float32 optimizer steps. These mechanisms also prevent weight decay
from changing frozen entries. We verify that explicitly frozen entries
are bit-identical before and after training.

\paragraph{The attention freeze.}
Freezing attention holds every attention parameter of every block: the
query, key, value and output projections, the query and key norms where
the model has them, and the attention biases where the model has them.
Freezing a key-value group holds the query rows and output columns of its
query heads, its own key and value rows, and the layer's query and key
norms. The block normalization gains train unless a configuration says
they are frozen.

\paragraph{Entry masks and calibration.}
A \emph{uniform-budget} mask freezes the same fraction of each
feed-forward matrix, usually $35.2\%$. A \emph{count-matched} mask
instead uses the number of entries above the magnitude threshold in
that matrix (Appendix~\ref{app:precision}). These counts can differ
between matrices. Every selected mask and its random control freeze
identical counts in each matrix. Other tested budgets are $20\%$ and
$50\%$, identified in the corresponding tables.
Drift-value freezes the highest-scoring entries, breaking ties by index.
Random masks sample uniformly using the run's seed.

The default calibration uses one float32 forward pass over 12 English
prompts, each truncated to 512 tokens. Excluding the first four positions
per prompt leaves 268 positions on Qwen3-1.7B, also used to compute
removal-value. The token count depends on the model's tokenizer.
The random-token control replaces token ids uniformly over the vocabulary
while preserving prompt lengths.

\paragraph{The compared scores.}
For a feed-forward matrix $y=Wx$, column $j$ contains all weights
$W_{ij}$ multiplying input feature $x_j$. Let $m_j=\mathbb E[x_j^2]$
be its second moment over calibration positions. The Wanda freezing
score is $|W_{ij}|\sqrt{m_j}$ \citep{wanda2023}.
The RIA freezing score is
\[
\left(\frac{|W_{ij}|}{\sum_k|W_{kj}|}
      +\frac{|W_{ij}|}{\sum_k|W_{ik}|}\right)m_j^{1/4},
\]
which scales each weight by its magnitude relative to its column and row
\citep{ria2024}. Both rules freeze the highest-scoring entries in each
matrix, ranked across the whole matrix as in Super-Tuning's layerwise
selection, rather than within each output row as Wanda prunes. SSU instead sums the Wanda scores within each input column,
$S_j=\sum_i |W_{ij}|\sqrt{m_j}$, and freezes the highest-scoring columns
whole, following SSU's formal column definition (its Section 3.2)
\citep{ssu2025}. For a target fraction $f$ and $d$ input columns,
it freezes $\operatorname{round}(fd)$ columns. Thus its achieved fraction
can differ slightly from the entry-wise rules.
The main Qwen3-1.7B comparison freezes $35.2051\%$ with SSU and
$35.2\%$ with the entry-wise rules.
These are score comparisons under our common attention-frozen training
protocol. The Wanda row is motivated by Super-Tuning's reuse of a pruning
score for selective fine-tuning \citep{supertuning2026}, and does not
represent a separate implementation of its full procedure.

\paragraph{The empirical Fisher baseline.}
For each of the $P$ calibration prompts, let $\mathcal L_p$ be the mean
next-token cross-entropy loss, with the first four target positions
excluded. The baseline score is
\[
F_{ij}^{\mathrm{prompt}}=\frac{1}{P}\sum_{p=1}^{P}
\left(\frac{\partial\mathcal L_p}{\partial W_{ij}}\right)^2.
\]
We take one backward pass per prompt, square each weight gradient, then
average over prompts. The highest-scoring entries are frozen at the same
count per matrix as drift-value. All rows and curves labeled ``Fisher''
or ``Empirical Fisher'' use this diagonal score unless explicitly labeled
as a variant. The factored variant is a separate control in
Appendix~\ref{app:drift}. No comparison computes a full Fisher matrix.

\paragraph{LoRA configuration.}
LoRA adapts the query, key, value, output, gate, up and down projections
in every block, keeping the backbone fixed. It uses zero dropout and
$\alpha=2r$, where $r$ is the rank. Ranks are 100 for Qwen2.5-1.5B,
42, 97, 130 and 199 for Qwen3-0.6B, 1.7B, 4B and 8B, respectively,
and 107 for SmolLM2-1.7B. They match the added-parameter budget of a
$10\%$ feed-forward expansion, rather than the trainable-parameter
count of entry freezing. Learning rates and acquisition differences are
reported in Appendix Table~\ref{tab:lora}.

\paragraph{Half fine-tuning.}
Our HFT control \citep{hft2024} freezes two of the three feed-forward
matrices in a random half of the blocks and one in the rest. This
matches a $50\%$ aggregate feed-forward budget. All attention parameters
remain frozen under the common comparison protocol.

\paragraph{Vision adaptation.}
DINOv3 ViT-S+ (29M parameters) \citep{dinov3_2025} learns ModelNet40
point clouds through a trainable tokenizer. Each cloud holds 1024 points
drawn uniformly at random from the 8192-point resampled release of the
ModelNet40 meshes \citep{modelnet2015}, centred and scaled to the unit
sphere, with the standard split of 9843 training and 2468 test objects. It groups 32 neighboring points
around each of 64 farthest-point-sampled centroids, pools each neighborhood with a shared
network, and projects it to the embedding width with a learned centroid
embedding. The tokenizer trains under every policy.
The parameter-matched adapter is one bottleneck adapter
\citep{houlsby2019adapters} after each feed-forward block, with its width
set per layer to match the parameters the annex adds there. Houlsby et al.
also place a second adapter after attention, and they keep one adapter
per task. Here a single adapter serves both modalities and stays active
when images are evaluated, so retention is measured on the adapted model.
Image retention and point-cloud acquisition use the $k$-nearest-neighbor
protocol in Appendix~\ref{app:settings}. Image chance accuracy is $0.01$.
Training minimizes cross-entropy through a learned classification head,
which is discarded for the $k$-nearest-neighbor evaluation.
Both protocols use batch size 64 and AdamW with weight decay $0.01$,
$(\beta_1,\beta_2)=(0.9,0.999)$ and $\epsilon=10^{-8}$.
Rates warm up linearly from $1\%$ of their peak over 50 steps, then
decay to zero on a cosine schedule. Standard training uses 10 epochs
and peak backbone rate $10^{-4}$. The stress protocol uses 30 epochs
and peak backbone rate $10^{-3}$. The tokenizer and classification head
use peak rate $10^{-3}$ in both protocols.
The stress protocol's main attention-only, random-mask and drift-value arms use
nine seeds. Additional controls use four, as listed in Appendix Table~\ref{tab:vision-budget}. Differences are paired by seed and use
$\max(0.031,2\,\mathrm{SE})$ as the decision threshold.

Attention freezing includes the query, key, value and output projections,
their preceding normalization and their layer scale in each block.
Entry masks freeze $20\%$, $35.2\%$ or $50\%$ of each feed-forward
matrix. Drift-value uses the same 256 calibration images as removal-value,
excluding class and register tokens. The random control matches its
counts. Training uses float32, with frozen entries restored after each
step. Feed-forward biases remain trainable.

\section{Rounding-induced freezing in bf16}
\label{app:precision}

Small updates can disappear when stored in bf16. This section explains
how that effect motivates the default frozen-entry budget, and why an
explicit mask is needed to compare selection rules independently of rounding.

\paragraph{An approximate magnitude threshold.}
With $p$ explicit mantissa bits, floating-point values in
$[2^n,2^{n+1})$ have spacing $2^{n-p}$. Away from binade boundaries,
rounding to nearest discards updates smaller than half this spacing
\citep{gupta2015,zamirai2020}. For a typical normalized optimizer step
of magnitude $\kappa\,\mathrm{lr}$, we estimate the threshold as
\begin{equation}
\tau=2^{\lceil\log_2(\kappa\,\mathrm{lr})\rceil+p+1}.
\label{eq:tau}
\end{equation}
Entries with $|w|\geq\tau$ are likely to remain unchanged at that step
size. This is an approximation, with exceptions at rounding ties and
binade boundaries. For bf16 ($p=7$) and $\kappa=1$, $\tau$ is $0.03125$
at learning rate $10^{-4}$, $0.015625$ at $5\times10^{-5}$, and
$0.0078125$ at $2\times10^{-5}$. At $10^{-4}$, the corresponding
float16 threshold is $0.25$ and the float32 threshold is $2048$.

\paragraph{Estimating the frozen fraction.}
The estimate is the fraction of weights above $\tau$. It can be counted
directly or approximated from a fitted magnitude distribution.
For a zero-centered generalized Gaussian with standard deviation $\sigma$, shape
$\beta$ and scale
$\alpha=\sigma\sqrt{\Gamma(1/\beta)/\Gamma(3/\beta)}$, the fraction is
$Q(1/\beta,(\tau/\alpha)^\beta)$, where $Q$ is the regularized upper
incomplete gamma function. Gaussian and Laplace special cases give
$\mathrm{erfc}(\tau/(\sigma\sqrt2))$ and
$\exp(-\sqrt2\tau/\sigma)$.
We fit $\beta$ from the ratio of mean absolute weight to $\sigma$.
This describes the bulk around $\tau$, whereas kurtosis emphasizes tails.
Across Qwen3-1.7B's 196 attention and feed-forward matrices, the estimated
fraction ranges from $24.3\%$ to $48.0\%$, containing $76\%$ to $93\%$
of squared weight norm. The fitted approximation differs from direct
counts by at most $0.011$, compared with $0.10$ for Gaussian or Laplace
fits and $0.07$ for a kurtosis-based fit. Fitted shapes range from
$1.2$ to $1.8$. These fits use weights only, without training outcomes.

\paragraph{Scope of the approximation.}
Adam's actual normalized step depends on gradient history through its
first and second moments. An entry above $\tau$ can therefore move when
its step is larger than the assumed $\kappa\,\mathrm{lr}$. Equation~\ref{eq:tau}
does not guarantee that such entries stay frozen throughout training.
AdamW's decay term, $\mathrm{lr}\,\lambda w$, and the direction of the
combined update also affect rounding. Under SGD, the step magnitude is
proportional to the gradient, so a weight-only threshold does not suffice.
Explicit masks, in contrast, guarantee unchanged weights.

\paragraph{Budgets motivated by rounding.}
At $10^{-4}$ and $\kappa=1$, the threshold selects $35.2\%$ of
Qwen3-1.7B's feed-forward entries. We use this as the
default comparison budget, with uniform and count-matched variants
defined in Appendix~\ref{app:cells}. The same rule selects $63.2\%$ on
Qwen3-1.7B at $5\times10^{-5}$, $77.4\%$ on
SmolLM2-1.7B at $10^{-4}$, $75.7\%$ on Qwen3-0.6B at
$2\times10^{-5}$, and $23.2\%$ on the pretrained Qwen3-8B at $10^{-4}$.
Thus the default budget is a convention across models, rather than each
model's predicted rounding-induced fraction. During bf16 repair with a
decaying learning rate, the estimated fraction increases, reaching
$84.3\%$ at the rate used in the final fifth of the tested
schedule.

\paragraph{Explicit magnitude masks.}
Rounding motivates both a budget and a preference for large-magnitude
weights. Stochastic rounding allows us to test this preference explicitly.
The magnitude mask improves retention over full fine-tuning by
$0.043\pm0.011$, but its differences from random masks of the same
counts are unresolved: $+0.009\pm0.018$ without attention freezing and
$-0.017\pm0.014$ with it.

\section{Drift-value}
\label{app:drift}

A frozen unit can still change its output because upstream parameters
change its input. This section measures that disturbance and derives the
entry score used to reduce it. The measured decomposition and the
forward-only score serve different purposes.

\paragraph{Input changes at a frozen unit.}
Freezing a unit's weights does not fix its inputs. For a feed-forward
unit $\nrn$ in block $\layer$ of a pre-norm transformer, the gate
pre-activation at token $t$ is
\begin{equation}
a_{\nrn}(t)=s_t\langle w_{\nrn}\odot g,h_t\rangle,\qquad
h_t=e_t+\sum_{k<\layer}(A_k(t)+M_k(t))+A_{\layer}(t).
\label{eq:reader}
\end{equation}
Here $w_{\nrn}$ is the gate row, $g$ the normalization gain vector,
$s_t$ the inverse root mean square, and $e_t$ the embedding.
$A_k$ and $M_k$ are the attention and feed-forward outputs added to the
residual stream. Even with $w_{\nrn}$ fixed, all other terms can change.
We decompose the pre-activation change into contributions from these
upstream outputs, normalization, and embeddings. Stacking changes over
frozen units and held-out text positions gives the drift vector $D$.
The vector $D_f$ contains the contribution from a particular family,
such as attention. Only the complete decomposition sums exactly to $D$.

\paragraph{From drift to retention loss.}
A local loss expansion motivates using squared drift as a measure of
disturbance. This requires a near-stationary old-task loss and approximately
isotropic curvature in drift coordinates. Neither assumption establishes
a law for discrete accuracy. We therefore fit an empirical response model:
\begin{equation}
\text{retention loss}\approx a+b\rho,\qquad
\rho=(1-q_f)(1-f)^{-\lambda},
\label{eq:law}
\end{equation}
where $f$ is the frozen entry fraction and $q_f$ its captured drift-value
share. The dimensionless ratio $\rho$ models the remaining disturbance.
The exponent $\lambda$ allows the updates of unfrozen entries to change
with the budget. The implementation averages these ratios using the
blocks' measured drift contributions.

On Qwen3-1.7B with stochastic rounding at $10^{-4}$, the random-mask
budget ladder gives $a=0.058\pm0.017$ and $b=0.170\pm0.025$ at
$\lambda=-0.45$. The exponent is selected using the drift-value mask
outcomes, so this fit is descriptive rather than an independent test.
The intercept is a fitted residual loss, not a measurement of pathways
that no mask can protect.

\paragraph{Freezing a family of upstream outputs.}
If freezing family $f$ removes $D_f$ while all other contributions stay
fixed, the remaining energy satisfies
\begin{equation}
\|D-D_f\|^2=\|D\|^2(1-2s_f+e_f),\qquad
s_f=\frac{\langle D_f,D\rangle}{\|D\|^2},\quad
e_f=\frac{\|D_f\|^2}{\|D\|^2}.
\label{eq:family}
\end{equation}
The projection share $s_f$ measures alignment with total drift, while
$e_f$ measures the family's own energy. This identity is exact. Its use
as a freezing prediction assumes that the other contributions do not respond.
Across nine seeds, attention has mean projection share $0.491$, trainable
feed-forward units $0.274$, and the frozen core's changed outputs $0.243$.
Their sum ranges from $0.996$ to $1.027$. The omitted normalization
contribution is about $-0.008$, with gain and embedding contributions
at $10^{-6}$ or less in the measured configuration.

For attention, $2s_f-e_f=0.434$, predicting removal of $43.4\%$ of
drift energy under the fixed-contribution assumption. Experimentally,
freezing attention beside the core reduces its retention loss by
$68.1\%\pm7.8\%$ in bf16 and $57.1\%\pm4.8\%$ with stochastic
rounding. The corresponding accuracy gains are $0.075\pm0.015$ and
$0.152\pm0.016$, both above the decision threshold.
The larger reductions are consistent with downstream contributions
shrinking when attention changes less.

Projection share alone cannot guarantee a benefit: Cauchy--Schwarz gives
$s_f^2\leq e_f$, but cancellation allows $e_f>1$ and $2s_f-e_f<0$.
Independent, equal-energy attention groups would predict half the benefit
from freezing a random half. The measured gain is $0.036\pm0.014$,
or $0.47$ of the all-attention gain in bf16.

\paragraph{Scoring individual entries.}
For $y=Wx$, an entry update $\delta$ at $(u,j)$ changes the output by
$\delta x_j e_u$. We approximate updates as having equal magnitude and
independent signs. This removes cross terms in expected squared drift.
We also factor the input second moment from the downstream squared gain,
giving the entry score
\begin{equation}
\pi_{uj}=\mathbb E[x_j^2]r_u^2.
\label{eq:price}
\end{equation}
This factorization is an approximation when input and downstream gain
vary together. The score measures local sensitivity to an update,
without multiplying by the entry's current weight.

In a gated feed-forward block, let $G_u$ and $U_u$ denote gate and
up-projection pre-activations, and $z_u=\mathrm{silu}(G_u)U_u$.
The up-projection squared gain is
$\mathbb E[\mathrm{silu}(G_u)^2]\|W_{\mathrm{down},:u}\|^2$.
The gate-projection gain is
$\mathbb E[(\mathrm{silu}'(G_u)U_u)^2]\|W_{\mathrm{down},:u}\|^2$.
A down-projection entry uses $\mathbb E[z_j^2]$, identical across output
rows. The tested DINOv3 ViT-S+ also uses gated SwiGLU blocks, so
the same formulas apply. All these quantities come from one forward pass without labels
or backward gradients.

\paragraph{Which Fisher scores we compare.}
For a single use of $y=Wx$, the weight gradient is $g_u x_j$, where
$g_u=\partial\mathcal L/\partial y_u$.
For a prompt, the weight gradient sums contributions across positions.
Our main empirical Fisher baseline squares that prompt gradient and then
averages over prompts, the empirical diagonal Fisher of \citet{fish2021} (Appendix~\ref{app:cells}).
EWC \citep{ewc2017} uses the diagonal Fisher as a parameter importance. With attention frozen
and matched entry counts, drift-value minus this baseline is
$-0.0089\pm0.0068$ on Qwen3-1.7B, inside the decision threshold over nine
seeds.

The separate \emph{factored Fisher} control uses
$\mathbb E[g_u^2]\mathbb E[x_j^2]$, averaging each moment over calibration
positions. It still requires backward gradients. Drift-value replaces
the gradient moment with the forward gain $r_u^2$.
The factored control also differs from the main baseline in how it
aggregates positions, so this comparison does not isolate factorization
alone. Neither score uses the off-diagonal entries of a full Fisher matrix.
Factored minus prompt-based Fisher is $-0.0028\pm0.0055$ on Qwen3-1.7B
and $+0.0033\pm0.0073$ on Qwen2.5-1.5B, both unresolved.
The factored control beats a random mask of the same size by
$0.0822\pm0.0083$ and $0.0650\pm0.0105$, respectively.

\paragraph{Relation to cooperative games.}
Removal-value approximates the average cost of removing a unit from
HOPE's capacity-based game (Appendix~\ref{app:proofs}). The drift
decomposition defines a different game: each player is an upstream
output contribution $d_i$, and a coalition has cost
$\|\sum_{i\in S}d_i\|^2$. Its Shapley value, the mean marginal cost
over insertion orders, is $\langle d_i,D\rangle$. Summing over a
family gives $s_f\|D\|^2$. These values sum to total drift energy when
all contributions are included. The entry score in Equation~\ref{eq:price}
instead uses the additive approximation described above.

\section{The screening rule}
\label{app:bound}

Selection can help only if some entries matter more than others. This
section combines the spread of drift-value with measured training effects
to estimate the benefit of choosing entries over freezing them at random.

Let $f$ be the fraction of entries frozen and $q_f$ the fraction of
summed drift-value they contain. A random mask captures $f$ in expectation.
A linear response model therefore estimates the selection gain as
\begin{equation}
\text{gain over random}\approx(q_f-f)\times\text{scale}.
\label{eq:bound}
\end{equation}
For the set with the largest possible $q_f$, this is a ceiling within
the model. It is not a bound on measured accuracy. Calibration supplies
$q_f$, but the scale requires training comparisons: random versus no
entry mask for entries, random core versus full fine-tuning for units,
and attention-frozen versus core-only training for heads.
These scales are empirical and specific to a configuration.

\begin{table}[!htbp]
\centering
\footnotesize
\caption{\textbf{Only entry selection gives a resolved gain in these comparisons.}
Screening estimates for Qwen3-1.7B use calibration scores and measured
response scales. The measured column gives each tested policy's gain
over its random control, paired by seed. Estimates are model predictions,
not certified accuracy bounds.}

\label{tab:bound}
\small
\input{figures/table7_bound}
\end{table}

\paragraph{Entries, heads and units.}
Entry drift-values have coefficient of variation $6.7$. Freezing $35.2\%$
captures $88.6\%$ to $96.2\%$ of drift-value, depending on aggregation.
Four scale and aggregation variants estimate gains of $0.066$ to $0.089$,
compared with $0.077\pm0.011$ measured (Appendix Table~\ref{tab:bound}).
The variants use raw or block-weighted shares and scales with or without
a correction for whether normalization gains train.

The highest-scoring half of key-value groups captures only $61\%$ of
head drift-value, giving estimated ceilings of $0.009$ to $0.017$.
The tested policy scores each head by its attention mass on retained
calibration positions times its output-projection Frobenius norm. Scores
are summed within each key-value group, and the highest-scoring half of
the groups in each layer are frozen. In bf16, alongside a $40\%$
removal-value unit core, this policy gains $-0.003\pm0.012$ over random
group selection. It differs from the drift-value-maximizing set used to
compute the ceiling.
For units, the $40\%$ removal-value core captures $93\%$ of unit
drift-value. Scales of $0.054$ in bf16 and $0.097$ with stochastic
rounding give estimates of $0.029$ and $0.052$, compared with measured
gains of $0.022\pm0.020$ and $0.013\pm0.020$. Both remain unresolved.
The removal-value core need not maximize the captured drift-value share,
so these are estimates for that core rather than universal ceilings.

\paragraph{Limits of the screen.}
Interactions between entries, changes in unfrozen updates, and accuracy
near chance or the pretrained level can invalidate linear scaling.
The response model in Appendix~\ref{app:drift} includes one such update
correction through $\lambda=-0.45$. Mask structure also matters:
whole-unit cores at $40\%$ and $35\%$ retain $0.030$ and $0.045$ less
than that model predicts from their captured shares.

\section{The headroom rule}
\label{app:headroom}

The screening rule uses both calibration scores and fitted response scales.
The headroom rule asks a simpler practical question: can retention after
attention freezing predict the benefit of an additional selected entry mask?

Remaining retention loss is pretrained retention minus retention after
attention freezing. An additional entry mask can recover some of this loss.
The random-mask gain measures the benefit of freezing a random set of entries.
The selection gain measures the additional benefit of choosing those entries
by drift-value. Appendix Table~\ref{tab:headroom} reports both paired gains for
fifteen configurations across six models, three precisions and four rates,
with nine seeds per policy. Five fix the slope, and the other ten are
out of sample, six of them on a model or a rate the fit never saw.

\begin{table}[!htbp]
\centering
\caption{\textbf{Remaining retention loss predicts the scope for entry selection.}
Fifteen configurations, nine paired seeds per policy. Five configurations
fit the slope and ten test it. Random-mask gain is relative to attention
freezing alone. Selection gain is drift-value minus random.
Bold gains exceed $\max(0.031,2\,\mathrm{SE})$.}

\label{tab:headroom}
\small
\input{figures/table8_headroom}
\end{table}

\begin{figure}[!htbp]
\centering
\includegraphics[width=0.68\textwidth]{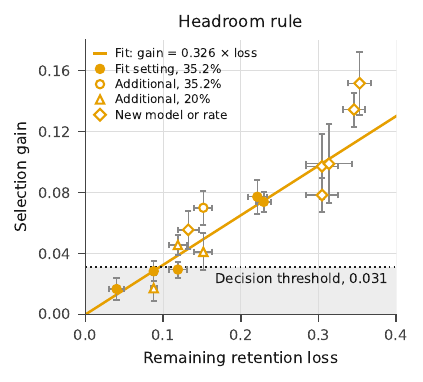}
\caption{\textbf{Remaining retention loss predicts selection benefit within the tested task.}
The line is fit to the five filled settings. Open markers show ten
additional settings. Circles freeze $35.2\%$ and triangles $20\%$.
Diamonds mark a model or rate absent from the fit at either fraction.
Error bars show one SE.
The dotted line is the minimum decision threshold, not a confidence bound.
The pale region falls below it.
Appendix Table~\ref{tab:headroom} identifies every setting.}
\label{fig:headroom}
\end{figure}

\paragraph{What is fitted.}
Let $L=R_0-R_A$, where $R_0$ is pretrained retention and $R_A$ is
retention after attention-frozen training. We fit
$R_{\mathrm{drift}}-R_{\mathrm{random}}=kL$ on five configurations at
$35.2\%$ frozen, marked in Appendix Table~\ref{tab:headroom}.
The result is $k=0.326\pm0.016$, with residual root mean square $0.005$
for gains ranging from $0.017$ to $0.077$. Allowing an intercept gives
$-0.001\pm0.006$, consistent with zero. A power-law fit performs worse.
On the same five configurations, random-mask gains have slope $0.36$.
A separate fit to the four $20\%$ configurations gives selection slope
$0.31$ and random-mask slope $0.21$. These random-mask slopes are close
to the corresponding frozen fractions.

\paragraph{Sensitivity to the fitted configurations.}
A through-origin fit gives more influence to larger values of $L$.
The two Qwen3-1.7B configurations at $10^{-4}$ contribute $81\%$ of
this influence, and all three Qwen3-1.7B configurations contribute $87\%$.
Pooling the two configurations that differ only in rounding gives
$k=0.319\pm0.021$. Leaving one configuration out gives $0.310$ to $0.336$.
Thus the fit is stable to these checks, but mostly determined by one model.

\paragraph{Using the prediction.}
Once $k$ is fitted, a new model needs pretrained retention and an
attention-frozen reference run to estimate $kL$. The predicted gain
reaches the minimum decision threshold at
$L=0.031/0.326=0.095\pm0.005$, where the uncertainty reflects the
fitted slope only. A measured gain must also exceed twice its paired SE.
The two configurations with $L<0.1$ have unresolved gains:
$0.028\pm0.006$ on Qwen3-1.7B at $5\times10^{-5}$ and
$0.017\pm0.007$ on Qwen3-0.6B in float32.
The rule estimates a gain over random selection. It does not compare
drift-value with Fisher, SSU or other selected masks.

\paragraph{Changing the frozen fraction.}
\label{app:count-correction}
The response model suggests scaling selection gain between budgets by
$(q_f-f)(1-f)^{-\lambda}$. With $q_{0.20}=0.787$,
$q_{0.352}=0.884$ and $\lambda=-0.45$, this predicts a factor of $1.214$
from $35.2\%$ to $20\%$. All four smaller-budget configurations fall
below this prediction. Their gains sum to $0.201$, versus $0.262$
predicted. Separate proportional fits describe both budgets better than
this correction transfers between them.

\paragraph{Configurations excluded from the fit.}
Ten configurations test the fitted slope, including six at an unseen
model or learning rate (Appendix Table~\ref{tab:headroom}). At $35.2\%$,
Qwen3-0.6B with stochastic rounding at $10^{-4}$ gains
$0.0700\pm0.0113$ over random, and SmolLM2-1.7B at $2\times10^{-4}$
gains $0.1344\pm0.0112$.
The previously unseen Qwen3-4B and 8B gain $0.1517\pm0.0206$ and
$0.0989\pm0.0259$, respectively. Their gains are
$0.43$ and $0.32$ of remaining loss. These results support approximate
proportionality within the tested task, without making $k$ a universal constant.

\section{Additional selection and adaptation controls}
\label{app:additional-controls}

These controls test whether the main findings depend on the baseline
implementation, frozen fraction, calibration data, training duration, or
numerical precision. The CNN controls concern removal-value. The
transformer controls test unit and entry selection.

\paragraph{Reproducing HOPE's VGG-8 transfer table.}
Before the connectivity intervention, we ran the CIFAR-100-to-SVHN
protocol of HOPE Table~2 over 20 local transfer scenarios. Our DEFT port
omits HOPE's merge step. It nevertheless reproduces the reported
retention and H-score closely: $52.09$ and $65.98$ versus HOPE's $52.14$
and $65.82$ \citep{hope2026}. Appendix Table~\ref{tab:hope-reproduction} compares
all five published methods with their local counterparts. Replacing only
DEFT's greedy selection score with removal-value raises H-score by
$3.94\pm0.39$ paired points, positive in all 20 scenarios, while changing
target accuracy by $-0.33$ points.
The runs use 30-epoch target training at each method's
validation-selected operating point. Retention restores the original
source head and BatchNorm statistics before test evaluation. The controlled
connectivity test in Figure~\ref{fig:regime} instead fixes the core fractions
and compares isolated with exposed cores, so its absolute scores are not
pooled with this reproduction.

\paragraph{CNN controls across frozen fractions.}
A separate VGG-8 repeat with paired training batches covers $40\%$, $55\%$ and $70\%$
frozen units. Every arm reaches its new-task target on all 20 transfers.
At the intermediate $55\%$ rung, isolated removal-value, DEFT and random
cores retain $43.81\%$, $36.65\%$ and $5.35\%$ old-task accuracy, while
the exposed removal-value and DEFT cores retain $5.74\%$ and
$5.94\%$. This repeat changes the batch streams and
normalization protocol, so its $55\%$ point is not pooled into
Figure~\ref{fig:regime}.

\paragraph{A second CNN transfer pair.}
The main VGG-8 panel uses transfers away from the source classes. On a
second, near-transfer pair with disjoint CIFAR classes, removal-value
improves H-score over DEFT's selection score at the same frozen fraction
by $2.10\pm0.38$ points across 20 paired scenarios (19 positive).
In a separate matched-acquisition comparison over five scenarios, it
exceeds a random core by $27.20\pm3.86$ H-score points at the $45\%$
new-task target and $29.66\pm3.56$ at $50\%$, but its margins over DEFT
are $0.84\pm1.02$ and $2.19\pm1.16$, inside the decision threshold.
Thus this pair supports selection over random, while the matched-acquisition
comparison does not resolve an advantage over DEFT.

\paragraph{Unit selection and the frozen fraction.}
On Qwen3-1.7B, the removal-value core differs from random and
DEFT cores by at most $0.024$ in paired retention at $40\%$ frozen across
float32, bf16, and stochastic rounding.
The comparisons at the lower learning rates also remain unresolved
(Appendix Table~\ref{tab:unit-policies}).
On Qwen3-1.7B in ordinary bf16, removal-value exceeds random cores by
$0.0422\pm0.0141$ at $70\%$ frozen and $0.0656\pm0.0144$ at $30\%$.
The corresponding stochastic-rounding comparison at $30\%$ ends near chance.
It cannot establish whether that advantage transfers across precisions.
For entry selection, Appendix Tables~\ref{tab:language-budget} and~\ref{tab:vision-budget}
report the $20\%$ and $50\%$ controls.
Appendix Table~\ref{tab:half-budget} compares HFT \citep{hft2024} and SSU at a
$50\%$ aggregate budget, with their structural differences explicit.
Appendix Table~\ref{tab:score-ablation} separates input moments, downstream gains,
and calibration text at the same mask budget.
\paragraph{Budget and model-size scope.}
A five-model sweep compares random, drift-value, SSU, Fisher, Wanda and
RIA at $20\%$, $35.2\%$ and $50\%$ of each feed-forward matrix.
At $50\%$, Qwen3-0.6B acquires $0.949$ with drift-value versus $0.989$
with a random mask, while their retentions are $0.446$ and
$0.448$. On the additional SmolLM2-360M model at
$35.2\%$, drift-value retains $0.538$ against random $0.529$, Fisher
$0.533$ and SSU $0.532$. Its acquisition is $0.976$ against random
$0.980$. This small-model comparison does not resolve a
selection gain. At $50\%$, its drift-value acquisition falls to
$0.964$ against random $0.976$. These controls delimit
the useful frozen fraction and model-size range.

\paragraph{Alternative sensitivity scores.}
Joint drift retains the correlation between input activity and downstream
gain that drift-value factors apart. Column drift aggregates this
sensitivity over whole input columns. A KL-weighted variant uses
backward probes to measure changes in the output distribution through
Kullback--Leibler divergence. These alternatives were compared with SSU
and Fisher over 27 seeds, checking acquisition and correcting for
multiple comparisons. None establishes a win over both
baselines. With identical 128-block calibration inputs,
KL-weighted drift retains $0.513$, Fisher $0.516$, and drift-value
$0.507$. These controls do not support replacing the
original score.

\paragraph{Calibration source and size.}
On Qwen3-1.7B, these controls test whether masks depend on meaningful
text and whether more calibration improves retention. They use blocks of 256
tokens from WikiText-103 \citep{wikitext2017} or uniformly sampled non-special token ids.
The four sizes are 8, 32, 128 and 512 blocks, with smaller samples nested
within larger ones. Three calibration draws and nine training seeds give
108 paired combinations per score. Pooled
text-minus-random-token retention differences are $-0.0021\pm0.0022$
for drift-value, $+0.0115\pm0.0020$ for Fisher,
$-0.0040\pm0.0021$ for joint drift, and $+0.0013\pm0.0017$ for SSU.
All are below the decision threshold. The reported SEs
describe the pooled differences. They do not account for dependence
from reused training seeds and nested calibration samples.

Mask stability is a separate question. We measure it by Jaccard overlap,
the intersection divided by the union of selected entries from two
calibration draws. Increasing text calibration from 8 to 512 blocks
raises drift-value overlap from $0.663$ to $0.933$, SSU from $0.644$
to $0.929$, joint drift from $0.627$ to $0.921$, and Fisher from
$0.438$ to $0.755$. Random-token calibration gives higher overlaps
for all four scores. However, more stable masks do not
establish better retention: across three corpus draws, increasing from
8 to 512 blocks changes drift-value retention from $0.4994$ to $0.5017$
and Fisher from $0.5102$ to $0.5120$, with draw-wise differences spanning
zero.

\paragraph{Longer adaptation.}
To test whether a mask computed before training remains useful, we extend
40-epoch language runs to 80 and 160 epochs (Appendix Figure~\ref{fig:shelf}).
On Qwen2.5-1.5B, drift-value minus Fisher grows from
$0.0022\pm0.0081$ at 40 epochs to $0.0433\pm0.0102$ at 160, with
acquisition $0.9970$ for both. On Qwen3-1.7B, their difference remains
unresolved at every duration. It also remains unresolved
at 160 epochs on Qwen3-0.6B at $10^{-4}$ and SmolLM2-1.7B at
$2\times10^{-4}$.
At 160 epochs, drift-value exceeds static SSU by $0.0383\pm0.0074$
on Qwen3-1.7B and $0.0489\pm0.0063$ on Qwen2.5-1.5B, with similar
acquisition. Differences from a Fisher-blend score, an equal mixture of
mean-normalized Fisher and joint-drift scores, remain unresolved on
both models.

\paragraph{Recomputing masks during training.}
A refreshed mask is recomputed at the current weights, allowing entries
to change between frozen and trainable states. On Qwen2.5-1.5B, refreshing
every 40 epochs improves Fisher over its static mask by
$0.0261\pm0.0054$ and drift-value by $0.0028\pm0.0036$, both below the
minimum threshold. The refreshed Fisher and static drift-value also
have an unresolved difference.
After the first 40 epochs, Fisher's recomputed mask has Jaccard overlap
$0.4709$ with its original mask, compared with $0.7452$ for drift-value.
The corresponding overlaps on Qwen3-1.7B are $0.5458$ and $0.7614$.
Both scores become more stable later (Appendix Figure~\ref{fig:shelf}).

Refreshing Fisher every 20 epochs improves its 160-epoch retention by
$0.0411\pm0.0056$ on Qwen2.5-1.5B, reaching $0.4411$ versus $0.4433$
for static drift-value. On Qwen3-1.7B the improvement is
$0.0133\pm0.0046$. The refreshed Fisher versus static drift-value
difference is unresolved on both models.
Thus the long-run comparison depends partly on when a score is computed.
Refreshing requires additional calibration and backward passes for Fisher.

\paragraph{LoRA with the same backbone precision.}
These controls reduce the precision mismatch in the main LoRA comparison.
Both methods use an ordinary bf16 backbone, while LoRA's trainable
adapter state remains float32. On Qwen2.5-1.5B, attention plus drift-value
retains $0.6250$ versus LoRA's $0.5717$, with acquisition $0.9889$ and
$0.9859$. The paired retention gain is $0.0533\pm0.0121$.
On Qwen3-1.7B, LoRA acquires only $0.1637$ versus $0.9903$ for the
freezing policy. On SmolLM2-1.7B, freezing gains $0.0456\pm0.0121$
in retention but loses $0.0963\pm0.0081$ in acquisition. Neither
comparison isolates retention at matched acquisition. SmolLM2's selected
and random masks also remain unresolved.

\begin{figure}[t]
\centering
\includegraphics[width=0.92\textwidth]{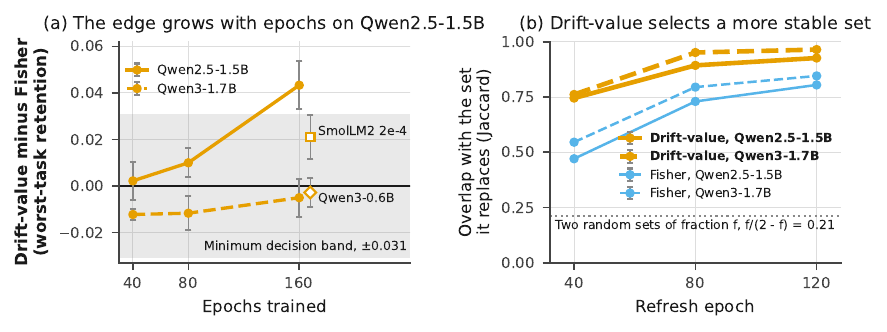}
\caption{\textbf{On Qwen2.5-1.5B drift-value pulls ahead of Fisher as adaptation runs longer, and its selected set changes less than Fisher's.}
(a) Drift-value minus Fisher on retention against epochs, nine seeds, means
$\pm$ SE, the $\pm0.031$ minimum decision band shaded. The hollow markers at 160 epochs
are SmolLM2-1.7B at $2\times10^{-4}$ and Qwen3-0.6B in
bf16 with stochastic rounding.
(b) Jaccard overlap of each score's recomputed set with the set it replaces at
epochs 40, 80 and 120. The dotted line is the approximate
expected overlap of two independent random sets of the same size,
$f/(2-f)=0.21$ at $f=0.352$.}
\label{fig:shelf}
\end{figure}

\paragraph{Selection beyond the largest entries.}
A control freezes the top $5\%$ by drift-value and fills
the remaining budget randomly.
On Qwen3-1.7B, this control trails the full drift-value ranking by $0.0317$,
with the same sign in all nine seeds.
On the vision stress experiment, it retains $0.286$ image accuracy,
compared with $0.446$ for full drift-value selection and $0.190$ for
random selection in this four-seed control.
The ranking below the largest entries also matters.
At three times the standard backbone rate, where attention freezing
leaves $0.074$ of image accuracy to recover, the drift-value mask's gain
over a random mask is $+0.009\pm0.010$ over four seeds, so on vision as
on language the gain follows the loss the attention freeze
leaves.

\paragraph{An additional retention task.}
To test retention beyond the three-task panel, we evaluate MMLU \citep{mmlu2021} on
Qwen3-1.7B using 200 items and nine paired seeds.
Drift-value exceeds its random control by $0.0311\pm0.0109$, just above
the $0.0310$ minimum threshold.

\begin{table}[!htbp]
\centering
\footnotesize
\caption{\textbf{Selection gains vary by model and follow the loss the attention freeze leaves.} Rows group models by family, then by increasing size, as in Figure~\ref{fig:stream}. With $35.2\%$ of every
feed-forward matrix frozen, stochastic rounding, and learning rate $10^{-4}$, nine seeds a
model. Retention is worst-task accuracy, and acquisition is fact recall.
The gain is drift-value minus random with its paired SE.
The Qwen3-1.7B row holds
the fixed $35.2\%$ of every matrix of Table~\ref{tab:races}, not the
count-matched mask of Appendix Tables~\ref{tab:bound}, \ref{tab:headroom}
and~\ref{tab:language-budget}.}
\label{tab:transfer-summary}
\input{figures/table5_transfer}
\end{table}

\begin{table}[!htbp]
\centering
\footnotesize
\caption{\textbf{The annex retains most, and drift-value beats a random mask of the
same count at all three budgets.} All tested policies on the vision
stress protocol of Figure~\ref{fig:vision}, means $\pm$ SE over the seeds
in the second column. Percentages beside entry policies are the share
of each feed-forward matrix frozen. Unit cores freeze approximately
$40\%$ of feed-forward units.}
\label{tab:vision-budget}
\input{figures/table_vision_budget}
\end{table}

\begin{table}[!htbp]
\centering\footnotesize
\caption{\textbf{The VGG-8 reproduction matches HOPE's reported DEFT result.
The removal-value core improves its H-score on the same local scenarios.}
Target accuracy, source retention and H-score are percentages. Published
means come from HOPE Table~2 \citep{hope2026}. Rep. denotes the local
reproduction over 20 paired CIFAR-100-to-SVHN transfers. Our DEFT port omits merging.
The removal-value (Ours) row changes only DEFT's selection score.
Its paired H-score gain over local DEFT is $3.94\pm0.39$ (SE).}
\label{tab:hope-reproduction}
\input{figures/table_hope_reproduction}
\end{table}

\begin{table}[!htbp]
\centering\footnotesize
\caption{\textbf{Isolation mainly changes source retention, while all policies reach the target.}
Complete comparison behind Figure~\ref{fig:regime}: means $\pm$ SE over
the same 20 transfers, in \%. H-score is the harmonic mean of retention
and acquisition.}
\label{tab:isolation}
\input{figures/table_isolation}
\end{table}

\begin{table}[!htbp]
\centering\footnotesize
\caption{\textbf{At the lower learning rates, removal-value, DEFT and random cores do not separate.}
Two Qwen3 models, nine seeds, means $\pm$ SE. Cores freeze $40\%$ of
feed-forward units, or $70\%$ where marked. Qwen3-0.6B uses float32 at
learning rate $2\times10^{-5}$, and Qwen3-1.7B bf16 SR at $5\times10^{-5}$.}
\label{tab:unit-policies}
\input{figures/table_unit_policies}
\end{table}

\begin{table}[!htbp]
\centering\small
\caption{\textbf{Drift-value exceeds random selection at each tested budget.}
Qwen3-1.7B, nine seeds and bf16 with stochastic rounding.
Both masks freeze attention. Retention and acquisition are means $\pm$ SE,
and the final column gives drift-value minus random at each matched budget.
Bold gains exceed $\max(0.031,2\,\mathrm{SE})$. Published-method comparisons appear in Table~\ref{tab:races} and Appendix Table~\ref{tab:half-budget}.}
\label{tab:language-budget}
\input{figures/table5_budget}
\end{table}
\FloatBarrier
\section{The pruning control}
\label{app:pruning-control}

Pruning checks the removal side of the distinction without adaptation
updates. Units are ranked by removal-value and attention is not pruned.
Density is the fraction of feed-forward units retained. We average
worst-task accuracy over seven densities from $0.95$ to $0.25$, using
three seeds per model and the evaluation in Appendix~\ref{app:settings}.
At uniform per-layer density, removal-value selection improves on random
selection by $0.016$ to $0.138$ across eight models from 360M to 32B, and
seven of the eight comparisons exceed twice their paired
SE.
Removal-value allocation from Equation~\ref{eq:phi} stays within $0.03$
of uniform allocation (Appendix Table~\ref{tab:pruning-controls}).
The separate repair experiments are specified in Appendix~\ref{app:settings}.
On a 26B mixture-of-experts model with 4B active parameters, when
deployment forces every expert to keep the same width, random picks beat
removal-value picks by up to $0.160$ at half density, so the
pruning results are scoped to dense, single-modality models.

\begin{table}[!htbp]
\centering\small
\caption{\textbf{Removal-value selection improves on random selection, and changing allocation gives smaller gains.}
Eight-model control study. The first difference holds uniform allocation
fixed, and the second holds the removal-value ranking fixed. Values are worst-task accuracy
differences averaged over seven densities, with SE across three seeds.
Bold selection differences exceed twice their paired SE.}
\label{tab:pruning-controls}
\input{figures/table_pruning_controls}
\end{table}

\FloatBarrier
\section{Further figures and tables}
\label{app:figures}

The following tables report the complete comparisons summarized in the
main figures. They separate score components, structural freezing rules,
and alternative adaptation methods. The final figure shows what the
annex learns while the original backbone stays frozen.

Appendix Tables~\ref{tab:score-ablation}--\ref{tab:lora} give these
comparisons, and Appendix Figure~\ref{fig:annex} shows the annex experiment.
At the standard rate, training the backbone with the annex reduces image
retention to $0.204$.

\begin{table}[!htbp]
\centering
\small
\caption{\textbf{Input-only and random-token variants have unresolved differences from drift-value. Row gain alone trails it.} Score ablations on Qwen3-1.7B at the count the
rounding rule of Appendix~\ref{app:precision} names. Retention and acquisition are means $\pm$ SE
across nine seeds under stochastic rounding. The final column is paired retention minus drift-value.}
\label{tab:score-ablation}
\input{figures/table5_ablation}
\end{table}

\begin{table}[!htbp]
\centering
\small
\caption{\textbf{Drift-value leads HFT and random entries at a $50\%$ budget. Its difference from SSU is unresolved.} Random entries, HFT, SSU, and drift-value at a $50\%$
feed-forward freezing budget,
with attention frozen. All use Qwen3-1.7B, bf16 with stochastic rounding,
and nine seeds. SSU freezes whole input columns. The HFT row freezes two of the three feed-forward matrices in a random
half of the blocks and one in each other block, so its budget is matched
in aggregate rather than in each matrix. HFT itself also trains two of
the four attention matrices in every block, which every row here holds
frozen.
Means $\pm$ SE. The final column is paired retention minus
drift-value.}
\label{tab:half-budget}
\input{figures/table5_halfbudget}
\end{table}

\begin{table}[!htbp]
\centering
\footnotesize
\caption{\textbf{With backbone precision matched, drift-value retains more than LoRA on Qwen3-0.6B and Qwen2.5-1.5B.}
Adaptation policies across six models. The final column is paired retention
minus attention plus drift-value. Bold differences exceed
$\max(0.031,2\,\mathrm{SE})$ in magnitude.
Qwen3-0.6B uses float32. The other initial blocks compare bf16 LoRA with
stochastic-rounding freezing, so precision and policy both differ.
The final three blocks use ordinary bf16 backbones for both methods.
LoRA adapters remain float32. Of these three, only Qwen2.5-1.5B has
matched acquisition. Rows with unmatched or unstable
acquisition omit retention differences, as defined in the final note.}
\label{tab:lora}
\begingroup
\setlength{\tabcolsep}{3pt}
\input{figures/table6_lora}
\endgroup
\end{table}

\begin{table}[t]
\centering
\footnotesize
\textbf{Appendix Table~\thetable\ (continued).}\par\smallskip
\begingroup
\setlength{\tabcolsep}{3pt}
\input{figures/table6_lora_cont}
\endgroup
\end{table}

\begin{figure}[!htbp]
\centering
\includegraphics[width=\textwidth]{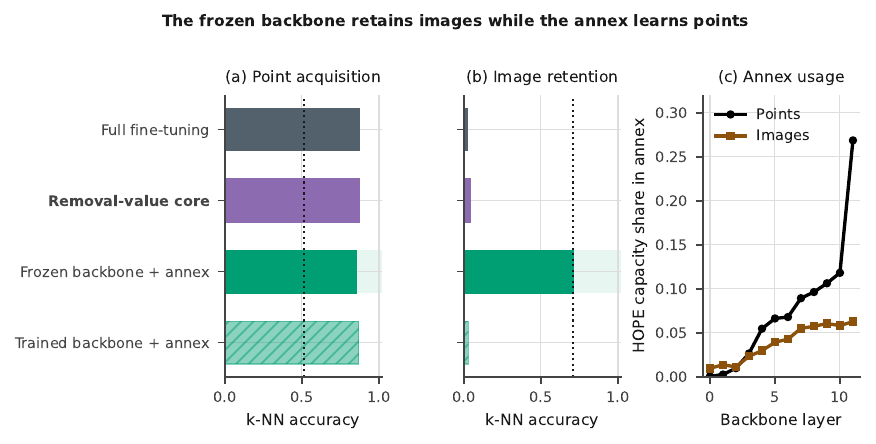}
\caption{\textbf{Freezing the backbone preserves image accuracy while the annex learns point clouds.} Single seed, stress protocol.
Figure~\ref{fig:vision} gives the comparison across multiple seeds. (a) Point-cloud
acquisition and (b) image retention, with accuracy before adaptation dotted.
(c) Annex units' share of HOPE's unit capacity (Equation~\ref{eq:capacity}),
measured separately on point-cloud and image inputs after training. Each
share divides the annex units' total capacity by that of all units in its
feed-forward layer. The pale row marks the frozen-backbone annex.}
\label{fig:annex}
\end{figure}

%% file: figures/table7_bound.tex
\resizebox{\textwidth}{!}{%
\begin{tabular}{p{3.8cm}ccl>{\columncolor{paperrow}}cl}
\toprule
Level & Frozen & \shortstack{Spread\\(CV)} & Model estimate & Measured gain & \shortstack{Gain vs.\\threshold} \\
\midrule
Attention heads & 0.500 & 0.30 & 0.009 [0.009, 0.017] & $-0.003\pm0.012$ (2/9) & unresolved \\
\textbf{Units, removal-value, bf16} & 0.400 & 2.05 & 0.029 & $+0.022\pm0.020$ (6/9) & unresolved \\
\textbf{Units, removal-value, bf16 SR} & 0.400 & 2.05 & 0.052 & $+0.013\pm0.020$ (5/9) & unresolved \\
\textbf{Entries, drift-value} & 0.352 & 6.74 & 0.078 [0.066, 0.089] & $\bm{+0.077\pm0.011}$ (9/9) & exceeds threshold \\
\bottomrule
\end{tabular}}
\par\smallskip
{\footnotesize Frozen is the fraction held. CV is the standard deviation of drift-value divided by its mean. Gains are paired over nine seeds, with positive seeds in parentheses. Measured gains use the threshold $\max(0.031,2\,\mathrm{SE})$. Brackets span four entry estimates or two head estimates using different shares and response scales. Estimates do not guarantee gains. The head estimate uses the top half by drift-value. The measured policy uses attention-window scores alongside a 40\% unit core. Entries use count-matched masks. Unit policies select by removal-value.}

%% file: figures/table8_headroom.tex
\resizebox{\textwidth}{!}{%
\begin{tabular}{llcccrc>{\columncolor{paperrow}}r}
\toprule
Model & Precision & Rate & Frozen & \shortstack{Remaining\\retention loss} & \shortstack{Random-mask\\gain} & \shortstack{Predicted\\selection gain} & \shortstack{Selection\\gain} \\
\midrule
Qwen3-4B$^{\ddagger}$ & bf16 SR & $10^{-4}$ & 0.352 & 0.353\,$\pm$\,0.015 & +0.100\,$\pm$\,0.024 (9/9) & +0.115 & $\bm{+0.152\pm0.021}$ (9/9) \\
SmolLM2-1.7B$^{\ddagger}$ & bf16 SR & $2\times10^{-4}$ & 0.352 & 0.346\,$\pm$\,0.014 & +0.042\,$\pm$\,0.005 (9/9) & +0.112 & $\bm{+0.134\pm0.011}$ (9/9) \\
Qwen3-8B$^{\ddagger}$ & bf16 SR & $10^{-4}$ & 0.352 & 0.313\,$\pm$\,0.030 & +0.132\,$\pm$\,0.032 (8/9) & +0.102 & $\bm{+0.099\pm0.026}$ (9/9) \\
Qwen2.5-1.5B$^{\ddagger}$ & bf16 SR & $10^{-4}$ & 0.352 & 0.304\,$\pm$\,0.020 & +0.108\,$\pm$\,0.011 (9/9) & +0.099 & $\bm{+0.078\pm0.011}$ (9/9) \\
Qwen3-1.7B$^{\dagger}$ & Kahan bf16 & $10^{-4}$ & 0.352 & 0.229\,$\pm$\,0.010 & +0.084\,$\pm$\,0.008 (9/9) & +0.075 & $\bm{+0.074\pm0.006}$ (9/9) \\
Qwen3-1.7B$^{\dagger}$ & bf16 SR & $10^{-4}$ & 0.352 & 0.221\,$\pm$\,0.012 & +0.067\,$\pm$\,0.012 (8/9) & +0.072 & $\bm{+0.077\pm0.011}$ (9/9) \\
Qwen3-0.6B$^{\ddagger}$ & bf16 SR & $10^{-4}$ & 0.352 & 0.152\,$\pm$\,0.012 & +0.031\,$\pm$\,0.014 (7/9) & +0.049 & $\bm{+0.070\pm0.011}$ (9/9) \\
Qwen2.5-1.5B$^{\ddagger}$ & bf16 SR & $5\times10^{-5}$ & 0.352 & 0.132\,$\pm$\,0.013 & +0.037\,$\pm$\,0.009 (9/9) & +0.043 & $\bm{+0.056\pm0.013}$ (9/9) \\
SmolLM2-1.7B$^{\dagger}$ & bf16 SR & $10^{-4}$ & 0.352 & 0.119\,$\pm$\,0.012 & +0.062\,$\pm$\,0.010 (9/9) & +0.039 & +0.029\,$\pm$\,0.005 (9/9) \\
Qwen3-1.7B$^{\dagger}$ & bf16 SR & $5\times10^{-5}$ & 0.352 & 0.088\,$\pm$\,0.004 & +0.029\,$\pm$\,0.009 (7/9) & +0.029 & +0.028\,$\pm$\,0.006 (8/9) \\
Qwen3-0.6B$^{\dagger}$ & float32 & $2\times10^{-5}$ & 0.352 & 0.040\,$\pm$\,0.009 & +0.024\,$\pm$\,0.007 (8/9) & +0.013 & +0.017\,$\pm$\,0.007 (6/9) \\
\midrule
Qwen2.5-1.5B$^{\ddagger}$ & bf16 SR & $10^{-4}$ & 0.2 & 0.304\,$\pm$\,0.020 & +0.059\,$\pm$\,0.010 (9/9) & +0.120 & $\bm{+0.097\pm0.021}$ (9/9) \\
Qwen3-0.6B$^{\ddagger}$ & bf16 SR & $10^{-4}$ & 0.2 & 0.152\,$\pm$\,0.012 & +0.022\,$\pm$\,0.013 (5/9) & +0.060 & $\bm{+0.041\pm0.012}$ (7/9) \\
SmolLM2-1.7B$^{\ddagger}$ & bf16 SR & $10^{-4}$ & 0.2 & 0.119\,$\pm$\,0.012 & +0.041\,$\pm$\,0.008 (9/9) & +0.047 & $\bm{+0.046\pm0.007}$ (9/9) \\
Qwen3-1.7B$^{\ddagger}$ & bf16 SR & $5\times10^{-5}$ & 0.2 & 0.088\,$\pm$\,0.004 & +0.027\,$\pm$\,0.007 (8/9) & +0.035 & +0.017\,$\pm$\,0.008 (7/9) \\
\bottomrule
\end{tabular}}
\par\smallskip
{\footnotesize Predicted selection gain is 0.326 times remaining retention loss. At a frozen fraction of 0.2, the prediction includes the factor 1.21 from Appendix~\ref{app:count-correction}. Frozen is the aggregate feed-forward entry fraction. Daggers mark the five fit configurations. Double daggers mark the ten excluded configurations (Appendix Figure~\ref{fig:headroom}). Parentheses give positive paired seeds out of all seeds.}

%% file: figures/table5_transfer.tex
\setlength{\tabcolsep}{3pt}
\begin{tabular}{@{}lrrrrr>{\columncolor{paperrow}}cc@{}}
\toprule
Model & Seeds & \shortstack{Pretrained\\retention} & \shortstack{Attention\\frozen} & \shortstack{Random\\mask} & \shortstack{\textbf{Drift-value}\\retention} & \shortstack{Selection\\gain} & \shortstack{\textbf{Drift-value}\\acquisition} \\
\midrule
Qwen2.5-1.5B & 9 & 0.642 & 0.338 & 0.446 & 0.524 & {\boldmath $+0.078\pm0.011$} & 0.989 \\
Qwen3-0.6B & 9 & 0.449 & 0.297 & 0.328 & 0.398 & {\boldmath $+0.070\pm0.011$} & 0.990 \\
Qwen3-1.7B & 9 & 0.577 & 0.356 & 0.426 & 0.502 & {\boldmath $+0.076\pm0.006$} & 0.990 \\
Qwen3-4B & 9 & 0.659 & 0.307 & 0.407 & 0.558 & {\boldmath $+0.152\pm0.021$} & 0.991 \\
Qwen3-8B & 9 & 0.716 & 0.402 & 0.534 & 0.633 & {\boldmath $+0.099\pm0.026$} & 0.988 \\
SmolLM2-1.7B & 9 & 0.664 & 0.545 & 0.607 & 0.636 & $+0.029\pm0.005$ & 0.982 \\
\bottomrule
\end{tabular}
\par\smallskip
{\footnotesize Bold differences exceed $\max(0.031,2\,\mathrm{SE})$ in absolute value.}

%% file: figures/table_vision_budget.tex
\begin{tabular}{lc>{\columncolor{paperrow}}cc}
\toprule
Policy & Seeds & Image retention & Point-cloud acquisition \\
\midrule
Full fine-tuning & 4 & $0.025\pm0.001$ & $0.871\pm0.008$ \\
\textbf{Removal-value core} & 4 & $0.036\pm0.002$ & $0.874\pm0.007$ \\
DEFT core & 4 & $0.044\pm0.006$ & $0.872\pm0.003$ \\
DEFT core, exposed & 4 & $0.030\pm0.005$ & $0.872\pm0.004$ \\
Attention frozen & 9 & $0.097\pm0.003$ & $0.887\pm0.003$ \\
Attention + random, 20.0\% & 4 & $0.127\pm0.015$ & $0.893\pm0.002$ \\
\textbf{Attention + drift-value, 20.0\%} & 4 & $0.316\pm0.041$ & $0.887\pm0.003$ \\
Attention + SSU columns, 20.0\% & 4 & $0.276\pm0.027$ & $0.880\pm0.006$ \\
Attention + Wanda score, 20.0\% & 4 & $0.166\pm0.015$ & $0.885\pm0.005$ \\
Attention + RIA score, 20.0\% & 4 & $0.158\pm0.006$ & $0.888\pm0.002$ \\
Attention + random, 35.2\% & 9 & $0.187\pm0.012$ & $0.888\pm0.004$ \\
\textbf{Attention + drift-value, 35.2\%} & 9 & $0.440\pm0.014$ & $0.882\pm0.003$ \\
Attention + SSU columns, 35.2\% & 4 & $0.344\pm0.049$ & $0.883\pm0.001$ \\
Attention + Wanda score, 35.2\% & 4 & $0.233\pm0.046$ & $0.884\pm0.005$ \\
Attention + RIA score, 35.2\% & 4 & $0.212\pm0.034$ & $0.887\pm0.003$ \\
Attention + random, 50.0\% & 4 & $0.291\pm0.023$ & $0.882\pm0.002$ \\
\textbf{Attention + drift-value, 50.0\%} & 4 & $0.549\pm0.026$ & $0.883\pm0.003$ \\
Attention + SSU columns, 50.0\% & 4 & $0.521\pm0.026$ & $0.876\pm0.006$ \\
Attention + Wanda score, 50.0\% & 4 & $0.358\pm0.047$ & $0.892\pm0.002$ \\
Attention + RIA score, 50.0\% & 4 & $0.334\pm0.024$ & $0.877\pm0.006$ \\
Bottleneck adapter & 4 & $0.373\pm0.034$ & $0.875\pm0.011$ \\
Frozen backbone + annex & 4 & $0.725\pm0.006$ & $0.868\pm0.007$ \\
\bottomrule
\end{tabular}
\par\smallskip
{\footnotesize Means and SEs use the seed count in the second column: 4 seeds for 19 arms and 9 seeds for 3 arms.}

%% file: figures/table_hope_reproduction.tex
\setlength{\tabcolsep}{4pt}
\begin{tabular*}{\textwidth}{@{}l@{\extracolsep{\fill}}rrrrr>{\columncolor{paperrow}}r@{}}
\toprule
 & \multicolumn{2}{c}{Target accuracy} & \multicolumn{2}{c}{Source retention} & \multicolumn{2}{c}{H-score} \\
\cmidrule(lr){2-3}\cmidrule(lr){4-5}\cmidrule(l){6-7}
Method & HOPE & Rep. & HOPE & Rep. & HOPE & Rep. \\
\midrule
DEFT & 89.79 & 90.85 & 52.14 & 52.09 & 65.82 & 65.98 \\
Head-only & 36.11 & 34.84 & 63.13 & 62.78 & 45.79 & 44.69 \\
Full fine-tuning & 94.09 & 93.66 & 7.52 & 6.04 & 13.88 & 11.31 \\
EWC & 93.94 & 93.69 & 6.74 & 6.09 & 12.54 & 11.41 \\
BN-only tuning & 81.91 & 83.05 & 5.44 & 6.11 & 10.18 & 11.36 \\
\midrule
\textbf{Removal-value (Ours)} & -- & 90.53 & -- & 57.29 & -- & 69.92 \\
\bottomrule
\end{tabular*}

%% file: figures/table_isolation.tex
\begin{tabular}{ll>{\columncolor{paperrow}}ccc}
\toprule
Frozen & Policy / connectivity & \begin{tabular}[b]{@{}c@{}}Retention\\(old-task acc., \%)\end{tabular} & \begin{tabular}[b]{@{}c@{}}Acquisition\\(new-task acc., \%)\end{tabular} & H-score \\
\midrule
40\% & Random core, exposed & $7.36\pm0.19$ & $82.16\pm0.32$ & $13.49\pm0.32$ \\
40\% & Random core, isolated & $5.03\pm0.03$ & $81.50\pm0.37$ & $9.48\pm0.05$ \\
40\% & DEFT core, exposed & $6.21\pm0.17$ & $81.01\pm0.43$ & $11.53\pm0.30$ \\
40\% & DEFT core, isolated & $13.88\pm0.90$ & $81.31\pm0.40$ & $23.49\pm1.31$ \\
40\% & \textbf{Removal-value core, exposed} & $6.67\pm0.23$ & $81.95\pm0.31$ & $12.32\pm0.40$ \\
40\% & \textbf{Removal-value core, isolated} & $18.54\pm1.42$ & $81.20\pm0.35$ & $29.66\pm1.93$ \\
70\% & Random core, exposed & $6.99\pm0.29$ & $81.72\pm0.31$ & $12.85\pm0.48$ \\
70\% & Random core, isolated & $8.69\pm0.82$ & $81.07\pm0.29$ & $15.46\pm1.28$ \\
70\% & DEFT core, exposed & $5.08\pm0.09$ & $81.37\pm0.40$ & $9.56\pm0.16$ \\
70\% & DEFT core, isolated & $52.11\pm1.52$ & $80.38\pm0.30$ & $62.98\pm1.18$ \\
70\% & \textbf{Removal-value core, exposed} & $5.58\pm0.18$ & $81.18\pm0.28$ & $10.43\pm0.31$ \\
70\% & \textbf{Removal-value core, isolated} & $57.33\pm1.61$ & $80.97\pm0.25$ & $66.87\pm1.16$ \\
\midrule
0\% & Gentle fine-tuning & $32.56\pm1.81$ & $80.98\pm0.22$ & $45.91\pm1.79$ \\
\bottomrule
\end{tabular}

%% file: figures/table_unit_policies.tex
\begin{tabular}{ll>{\columncolor{paperrow}}cc}
\toprule
Model & Policy & Retention & Acquisition \\
\midrule
Qwen3-0.6B & Full fine-tuning & $0.356\pm0.012$ & $0.988\pm0.004$ \\
 & \textbf{Removal-value core} & $0.384\pm0.010$ & $0.991\pm0.003$ \\
 & DEFT core & $0.377\pm0.012$ & $0.990\pm0.004$ \\
 & Random core & $0.383\pm0.013$ & $0.987\pm0.003$ \\
 & \textbf{Removal-value core, 70\%} & $0.393\pm0.013$ & $0.958\pm0.028$ \\
 & Random core, 70\% & $0.389\pm0.010$ & $0.992\pm0.002$ \\
 & \textbf{Removal-value core + attention frozen} & $0.441\pm0.009$ & $0.977\pm0.008$ \\
\midrule
Qwen3-1.7B & Full fine-tuning & $0.405\pm0.016$ & $0.990\pm0.003$ \\
 & \textbf{Removal-value core} & $0.448\pm0.012$ & $0.988\pm0.004$ \\
 & DEFT core & $0.464\pm0.012$ & $0.988\pm0.004$ \\
 & Random core & $0.442\pm0.012$ & $0.993\pm0.004$ \\
 & \textbf{Removal-value core, 70\%} & $0.484\pm0.010$ & $0.986\pm0.004$ \\
 & Random core, 70\% & $0.480\pm0.010$ & $0.991\pm0.003$ \\
 & \textbf{Removal-value core + attention frozen} & $0.537\pm0.008$ & $0.990\pm0.003$ \\
\bottomrule
\end{tabular}

%% file: figures/table5_budget.tex
\setlength{\tabcolsep}{3pt}
\begin{tabular}{@{}lcccc>{\columncolor{paperrow}}c@{}}
\toprule
Frozen feed- & \multicolumn{2}{c}{Random mask} & \multicolumn{2}{c}{\textbf{Drift-value mask}} & Paired \\
forward entries & \begin{tabular}[b]{@{}c@{}}Retention\\(worst-task acc.)\end{tabular} & \begin{tabular}[b]{@{}c@{}}Acquisition\\(fact recall)\end{tabular} & \begin{tabular}[b]{@{}c@{}}Retention\\(worst-task acc.)\end{tabular} & \begin{tabular}[b]{@{}c@{}}Acquisition\\(fact recall)\end{tabular} & \begin{tabular}[b]{@{}c@{}}gain in\\retention\end{tabular} \\
\midrule
20\% & $0.388\pm0.015$ & $0.985\pm0.004$ & $0.474\pm0.011$ & $0.990\pm0.003$ & {\boldmath $+0.086\pm0.013$} \\
35.2\% & $0.423\pm0.008$ & $0.990\pm0.002$ & $0.500\pm0.008$ & $0.991\pm0.002$ & {\boldmath $+0.077\pm0.011$} \\
50\% & $0.471\pm0.012$ & $0.992\pm0.002$ & $0.518\pm0.006$ & $0.990\pm0.003$ & {\boldmath $+0.047\pm0.012$} \\
\bottomrule
\end{tabular}

%% file: figures/table_pruning_controls.tex
\begin{tabular}{l>{\columncolor{paperrow}}cc}
\toprule
Model & \textbf{Removal-value} minus random & \textbf{Removal-value} minus uniform allocation \\
\midrule
Mistral-7B & $\bm{+0.138\pm0.024}$ & $-0.029\pm0.003$ \\
OLMo-2-7B & $\bm{+0.109\pm0.020}$ & $+0.002\pm0.003$ \\
Qwen3-0.6B & $\bm{+0.017\pm0.006}$ & $+0.009\pm0.003$ \\
Qwen3-1.7B & $+0.016\pm0.013$ & $+0.021\pm0.004$ \\
Qwen3-4B & $\bm{+0.043\pm0.002}$ & $+0.016\pm0.001$ \\
Qwen3-8B & $\bm{+0.030\pm0.003}$ & $-0.016\pm0.009$ \\
Qwen3-32B & $\bm{+0.089\pm0.030}$ & $+0.004\pm0.003$ \\
SmolLM2-360M & $\bm{+0.039\pm0.007}$ & $-0.011\pm0.002$ \\
\bottomrule
\end{tabular}

%% file: figures/table5_ablation.tex
\setlength{\tabcolsep}{4pt}
\begin{tabular*}{\textwidth}{@{}l@{\extracolsep{\fill}}cc>{\columncolor{paperrow}}c@{}}
\toprule
Policy / selection rule & \begin{tabular}[b]{@{}c@{}}Retention\\(worst-task accuracy)\end{tabular} & \begin{tabular}[b]{@{}c@{}}Acquisition\\(fact recall)\end{tabular} & \begin{tabular}[b]{@{}c@{}}$\Delta$ retention\\vs \textbf{drift-value}\end{tabular} \\
\midrule
Pretrained model & $0.577\pm0.005$ & -- & -- \\
Output row gain alone & $0.466\pm0.009$ & $0.986\pm0.005$ & {\boldmath $-0.034\pm0.007$} \\
Input moment alone & $0.492\pm0.010$ & $0.990\pm0.002$ & $-0.008\pm0.007$ \\
\textbf{Drift-value} & $0.500\pm0.008$ & $0.991\pm0.002$ & reference \\
\textbf{Drift-value, random tokens} & $0.503\pm0.006$ & $0.986\pm0.004$ & $+0.003\pm0.007$ \\
\bottomrule
\end{tabular*}
\par\smallskip
{\footnotesize Bold differences exceed $\max(0.031,2\,\mathrm{SE})$ in absolute value.}

%% file: figures/table5_halfbudget.tex
\setlength{\tabcolsep}{4pt}
\begin{tabular*}{\textwidth}{@{}l@{\extracolsep{\fill}}cc>{\columncolor{paperrow}}c@{}}
\toprule
Policy / selection rule & \begin{tabular}[b]{@{}c@{}}Retention\\(worst-task accuracy)\end{tabular} & \begin{tabular}[b]{@{}c@{}}Acquisition\\(fact recall)\end{tabular} & \begin{tabular}[b]{@{}c@{}}$\Delta$ retention\\vs \textbf{drift-value}\end{tabular} \\
\midrule
Pretrained model & $0.577\pm0.005$ & -- & -- \\
Random mask (control) & $0.471\pm0.012$ & $0.992\pm0.002$ & {\boldmath $-0.047\pm0.012$} \\
Half fine-tuning (HFT) & $0.474\pm0.015$ & $0.988\pm0.003$ & {\boldmath $-0.043\pm0.013$} \\
SSU columns & $0.504\pm0.008$ & $0.990\pm0.004$ & $-0.014\pm0.005$ \\
\textbf{Drift-value} & $0.518\pm0.006$ & $0.990\pm0.003$ & reference \\
\bottomrule
\end{tabular*}
\par\smallskip
{\footnotesize Bold differences exceed $\max(0.031,2\,\mathrm{SE})$ in absolute value.}

%% file: figures/table6_lora.tex
\begin{tabularx}{\textwidth}{@{}>{\raggedright\arraybackslash}Xccccc>{\columncolor{paperrow}}c@{}}
\toprule
Policy & Precision & LR & \shortstack{Retention\\(worst-task\\accuracy)} & \shortstack{Change from\\pretrained\\accuracy} & \shortstack{Acquisition\\(fact recall)} & \shortstack{$\Delta$ retention\\vs attention +\\\textbf{drift-value}} \\
\midrule
\multicolumn{7}{l}{\textit{Qwen2.5-1.5B}} \\
Full fine-tuning & bf16 SR & 1e-4 & $0.306{\pm}0.009$ & $-0.336{\pm}0.013$ & $0.986{\pm}0.003$ & {\boldmath $-0.218{\pm}0.008$} \\
DEFT core & bf16 SR & 1e-4 & $0.336{\pm}0.010$ & $-0.307{\pm}0.018$ & $0.991{\pm}0.003$ & {\boldmath $-0.189{\pm}0.011$} \\
SSU columns & bf16 SR & 1e-4 & $0.503{\pm}0.007$ & $-0.139{\pm}0.011$ & $0.990{\pm}0.004$ & $-0.021{\pm}0.004$ \\
Empirical Fisher & bf16 SR & 1e-4 & $0.508{\pm}0.010$ & $-0.134{\pm}0.016$ & $0.988{\pm}0.003$ & $-0.017{\pm}0.009$ \\
Wanda freezing score & bf16 SR & 1e-4 & $0.469{\pm}0.010$ & $-0.173{\pm}0.016$ & $0.986{\pm}0.004$ & {\boldmath $-0.055{\pm}0.007$} \\
RIA freezing score & bf16 SR & 1e-4 & $0.445{\pm}0.013$ & $-0.197{\pm}0.018$ & $0.990{\pm}0.003$ & {\boldmath $-0.079{\pm}0.012$} \\
Attention frozen & bf16 SR & 1e-4 & $0.338{\pm}0.017$ & $-0.304{\pm}0.020$ & $0.983{\pm}0.005$ & {\boldmath $-0.187{\pm}0.012$} \\
Attention + random & bf16 SR & 1e-4 & $0.446{\pm}0.015$ & $-0.196{\pm}0.017$ & $0.984{\pm}0.005$ & {\boldmath $-0.078{\pm}0.011$} \\
\textbf{Attention + drift-value} & bf16 SR & 1e-4 & $0.524{\pm}0.007$ & $-0.118{\pm}0.011$ & $0.989{\pm}0.004$ & reference \\
LoRA & bf16 & 1e-4 & $0.572{\pm}0.011$ & $-0.071{\pm}0.011$ & $0.986{\pm}0.005$ & {\boldmath $+0.047{\pm}0.011$} \\
\midrule
\multicolumn{7}{l}{\textit{Qwen3-0.6B}} \\
Full fine-tuning & fp32 & 2e-5 & $0.356{\pm}0.012$ & $-0.092{\pm}0.012$ & $0.988{\pm}0.004$ & {\boldmath $-0.092{\pm}0.007$} \\
DEFT core & fp32 & 2e-5 & $0.377{\pm}0.012$ & $-0.071{\pm}0.010$ & $0.990{\pm}0.004$ & {\boldmath $-0.071{\pm}0.007$} \\
SSU columns & fp32 & 2e-5 & $0.449{\pm}0.010$ & $0.001{\pm}0.006$ & $0.980{\pm}0.003$ & $+0.001{\pm}0.004$ \\
Empirical Fisher & fp32 & 2e-5 & $0.448{\pm}0.011$ & $0.001{\pm}0.005$ & $0.965{\pm}0.015$ & $-0.000{\pm}0.003$ \\
Wanda freezing score & fp32 & 2e-5 & $0.438{\pm}0.010$ & $-0.009{\pm}0.009$ & $0.987{\pm}0.004$ & $-0.010{\pm}0.005$ \\
RIA freezing score & fp32 & 2e-5 & $0.439{\pm}0.010$ & $-0.008{\pm}0.009$ & $0.991{\pm}0.002$ & $-0.009{\pm}0.005$ \\
Attention frozen & fp32 & 2e-5 & $0.408{\pm}0.012$ & $-0.040{\pm}0.009$ & $0.979{\pm}0.013$ & {\boldmath $-0.041{\pm}0.009$} \\
Attention + random & fp32 & 2e-5 & $0.432{\pm}0.006$ & $-0.016{\pm}0.010$ & $0.987{\pm}0.005$ & $-0.017{\pm}0.007$ \\
\textbf{Attention + drift-value} & fp32 & 2e-5 & $0.448{\pm}0.010$ & $0.001{\pm}0.007$ & $0.973{\pm}0.006$ & reference \\
LoRA & fp32 & 1e-4 & $0.384{\pm}0.008$ & $-0.063{\pm}0.011$ & $0.989{\pm}0.003$ & {\boldmath $-0.064{\pm}0.010$} \\
LoRA & fp32 & 2e-5 & $0.444{\pm}0.005$ & $-0.004{\pm}0.012$ & $0.705{\pm}0.020$ & not matched \\
\midrule
\multicolumn{7}{l}{\textit{Qwen3-1.7B}} \\
Full fine-tuning & bf16 SR & 1e-4 & $0.263{\pm}0.010$ & $-0.314{\pm}0.009$ & $0.988{\pm}0.004$ & {\boldmath $-0.239{\pm}0.006$} \\
DEFT core & bf16 SR & 1e-4 & $0.312{\pm}0.016$ & $-0.266{\pm}0.019$ & $0.911{\pm}0.076$ & variable acq. \\
SSU columns & bf16 SR & 1e-4 & $0.487{\pm}0.007$ & $-0.090{\pm}0.007$ & $0.988{\pm}0.004$ & $-0.015{\pm}0.006$ \\
Empirical Fisher & bf16 SR & 1e-4 & $0.511{\pm}0.006$ & $-0.066{\pm}0.008$ & $0.989{\pm}0.003$ & $+0.009{\pm}0.007$ \\
Wanda freezing score & bf16 SR & 1e-4 & $0.458{\pm}0.011$ & $-0.119{\pm}0.009$ & $0.987{\pm}0.003$ & {\boldmath $-0.044{\pm}0.005$} \\
RIA freezing score & bf16 SR & 1e-4 & $0.429{\pm}0.009$ & $-0.148{\pm}0.009$ & $0.990{\pm}0.003$ & {\boldmath $-0.073{\pm}0.009$} \\
Attention frozen & bf16 SR & 1e-4 & $0.356{\pm}0.011$ & $-0.221{\pm}0.012$ & $0.990{\pm}0.003$ & {\boldmath $-0.146{\pm}0.009$} \\
Attention + random & bf16 SR & 1e-4 & $0.426{\pm}0.008$ & $-0.151{\pm}0.008$ & $0.994{\pm}0.003$ & {\boldmath $-0.076{\pm}0.006$} \\
\textbf{Attention + drift-value} & bf16 SR & 1e-4 & $0.502{\pm}0.007$ & $-0.075{\pm}0.005$ & $0.990{\pm}0.004$ & reference \\
LoRA & bf16 & 5e-4 & $0.241{\pm}0.006$ & $-0.337{\pm}0.009$ & $0.164{\pm}0.024$ & not matched \\
LoRA & bf16 & 1e-4 & $0.513{\pm}0.011$ & $-0.064{\pm}0.011$ & $0.990{\pm}0.004$ & $+0.011{\pm}0.012$ \\
\midrule
\multicolumn{7}{l}{\textit{Qwen3-4B}} \\
Full fine-tuning & bf16 SR & 1e-4 & $0.252{\pm}0.009$ & $-0.407{\pm}0.010$ & $0.948{\pm}0.042$ & {\boldmath $-0.306{\pm}0.019$} \\
DEFT core & bf16 SR & 1e-4 & $0.243{\pm}0.007$ & $-0.416{\pm}0.007$ & $0.889{\pm}0.058$ & variable acq. \\
SSU columns & bf16 SR & 1e-4 & $0.524{\pm}0.015$ & $-0.136{\pm}0.015$ & $0.987{\pm}0.004$ & {\boldmath $-0.034{\pm}0.009$} \\
Empirical Fisher & bf16 SR & 1e-4 & $0.545{\pm}0.015$ & $-0.114{\pm}0.017$ & $0.993{\pm}0.002$ & $-0.013{\pm}0.010$ \\
Wanda freezing score & bf16 SR & 1e-4 & $0.453{\pm}0.017$ & $-0.206{\pm}0.017$ & $0.981{\pm}0.008$ & {\boldmath $-0.105{\pm}0.012$} \\
RIA freezing score & bf16 SR & 1e-4 & $0.451{\pm}0.019$ & $-0.209{\pm}0.020$ & $0.987{\pm}0.002$ & {\boldmath $-0.108{\pm}0.015$} \\
Attention frozen & bf16 SR & 1e-4 & $0.307{\pm}0.018$ & $-0.353{\pm}0.015$ & $0.973{\pm}0.013$ & {\boldmath $-0.252{\pm}0.015$} \\
Attention + random & bf16 SR & 1e-4 & $0.407{\pm}0.030$ & $-0.253{\pm}0.028$ & $0.988{\pm}0.004$ & {\boldmath $-0.152{\pm}0.021$} \\
\textbf{Attention + drift-value} & bf16 SR & 1e-4 & $0.558{\pm}0.013$ & $-0.101{\pm}0.011$ & $0.991{\pm}0.004$ & reference \\
LoRA & bf16 & 1e-4 & $0.536{\pm}0.032$ & $-0.124{\pm}0.028$ & $0.990{\pm}0.005$ & $-0.023{\pm}0.030$ \\
\bottomrule
\end{tabularx}

%% file: figures/table6_lora_cont.tex
\begin{tabularx}{\textwidth}{@{}>{\raggedright\arraybackslash}Xccccc>{\columncolor{paperrow}}c@{}}
\toprule
Policy & Precision & LR & \shortstack{Retention\\(worst-task\\accuracy)} & \shortstack{Change from\\pretrained\\accuracy} & \shortstack{Acquisition\\(fact recall)} & \shortstack{$\Delta$ retention\\vs attention +\\\textbf{drift-value}} \\
\midrule
\multicolumn{7}{l}{\textit{Qwen3-8B}} \\
Full fine-tuning & bf16 SR & 1e-4 & $0.271{\pm}0.016$ & $-0.445{\pm}0.016$ & $0.858{\pm}0.089$ & variable acq. \\
DEFT core & bf16 SR & 1e-4 & $0.348{\pm}0.031$ & $-0.367{\pm}0.033$ & $0.989{\pm}0.002$ & {\boldmath $-0.285{\pm}0.033$} \\
SSU columns & bf16 SR & 1e-4 & $0.613{\pm}0.007$ & $-0.103{\pm}0.007$ & $0.990{\pm}0.004$ & $-0.021{\pm}0.010$ \\
Empirical Fisher & bf16 SR & 1e-4 & $0.632{\pm}0.009$ & $-0.083{\pm}0.009$ & $0.993{\pm}0.003$ & $-0.001{\pm}0.009$ \\
Wanda freezing score & bf16 SR & 1e-4 & $0.551{\pm}0.036$ & $-0.164{\pm}0.038$ & $0.989{\pm}0.004$ & $-0.082{\pm}0.042$ \\
RIA freezing score & bf16 SR & 1e-4 & $0.586{\pm}0.008$ & $-0.130{\pm}0.010$ & $0.990{\pm}0.002$ & {\boldmath $-0.048{\pm}0.010$} \\
Attention frozen & bf16 SR & 1e-4 & $0.402{\pm}0.026$ & $-0.313{\pm}0.030$ & $0.992{\pm}0.003$ & {\boldmath $-0.231{\pm}0.027$} \\
Attention + random & bf16 SR & 1e-4 & $0.534{\pm}0.024$ & $-0.181{\pm}0.025$ & $0.990{\pm}0.004$ & {\boldmath $-0.099{\pm}0.026$} \\
\textbf{Attention + drift-value} & bf16 SR & 1e-4 & $0.633{\pm}0.011$ & $-0.082{\pm}0.011$ & $0.988{\pm}0.004$ & reference \\
LoRA & bf16 & 1e-4 & $0.525{\pm}0.048$ & $-0.191{\pm}0.044$ & $0.890{\pm}0.097$ & variable acq. \\
\midrule
\multicolumn{7}{l}{\textit{SmolLM2-1.7B}} \\
Full fine-tuning & bf16 SR & 1e-4 & $0.489{\pm}0.006$ & $-0.174{\pm}0.011$ & $0.993{\pm}0.003$ & {\boldmath $-0.147{\pm}0.009$} \\
DEFT core & bf16 SR & 1e-4 & $0.543{\pm}0.008$ & $-0.121{\pm}0.007$ & $0.989{\pm}0.002$ & {\boldmath $-0.093{\pm}0.009$} \\
SSU columns & bf16 SR & 1e-4 & $0.637{\pm}0.010$ & $-0.027{\pm}0.008$ & $0.987{\pm}0.005$ & $+0.001{\pm}0.004$ \\
Empirical Fisher & bf16 SR & 1e-4 & $0.633{\pm}0.010$ & $-0.031{\pm}0.008$ & $0.987{\pm}0.004$ & $-0.003{\pm}0.004$ \\
Wanda freezing score & bf16 SR & 1e-4 & $0.629{\pm}0.010$ & $-0.034{\pm}0.009$ & $0.987{\pm}0.005$ & $-0.007{\pm}0.005$ \\
RIA freezing score & bf16 SR & 1e-4 & $0.624{\pm}0.009$ & $-0.040{\pm}0.007$ & $0.990{\pm}0.004$ & $-0.012{\pm}0.005$ \\
Attention frozen & bf16 SR & 1e-4 & $0.545{\pm}0.009$ & $-0.119{\pm}0.012$ & $0.990{\pm}0.003$ & {\boldmath $-0.091{\pm}0.010$} \\
Attention + random & bf16 SR & 1e-4 & $0.607{\pm}0.011$ & $-0.057{\pm}0.010$ & $0.992{\pm}0.003$ & $-0.029{\pm}0.005$ \\
\textbf{Attention + drift-value} & bf16 SR & 1e-4 & $0.636{\pm}0.011$ & $-0.028{\pm}0.009$ & $0.982{\pm}0.005$ & reference \\
LoRA & bf16 & 1e-4 & $0.624{\pm}0.008$ & $-0.040{\pm}0.012$ & $0.986{\pm}0.004$ & $-0.012{\pm}0.010$ \\
\midrule
\multicolumn{7}{l}{\textit{Qwen2.5-1.5B, one precision}} \\
\textbf{Attention + drift-value} & bf16 & 1e-4 & $0.625{\pm}0.013$ & $-0.017{\pm}0.009$ & $0.989{\pm}0.004$ & reference \\
Attention + random & bf16 & 1e-4 & $0.560{\pm}0.013$ & $-0.082{\pm}0.016$ & $0.991{\pm}0.004$ & {\boldmath $-0.065{\pm}0.014$} \\
Attention frozen & bf16 & 1e-4 & $0.504{\pm}0.006$ & $-0.138{\pm}0.011$ & $0.988{\pm}0.005$ & {\boldmath $-0.121{\pm}0.016$} \\
LoRA & bf16 & 1e-4 & $0.572{\pm}0.011$ & $-0.071{\pm}0.011$ & $0.986{\pm}0.005$ & {\boldmath $-0.053{\pm}0.012$} \\
\midrule
\multicolumn{7}{l}{\textit{Qwen3-1.7B, one precision}} \\
\textbf{Attention + drift-value} & bf16 & 1e-4 & $0.564{\pm}0.007$ & $-0.013{\pm}0.005$ & $0.990{\pm}0.003$ & reference \\
Attention + random & bf16 & 1e-4 & $0.536{\pm}0.010$ & $-0.041{\pm}0.008$ & $0.989{\pm}0.004$ & $-0.028{\pm}0.010$ \\
Attention frozen & bf16 & 1e-4 & $0.498{\pm}0.012$ & $-0.079{\pm}0.011$ & $0.989{\pm}0.002$ & {\boldmath $-0.066{\pm}0.014$} \\
LoRA & bf16 & 5e-4 & $0.241{\pm}0.006$ & $-0.337{\pm}0.009$ & $0.164{\pm}0.024$ & not matched \\
\midrule
\multicolumn{7}{l}{\textit{SmolLM2-1.7B, one precision}} \\
\textbf{Attention + drift-value} & bf16 & 1e-4 & $0.669{\pm}0.010$ & $0.006{\pm}0.002$ & $0.890{\pm}0.010$ & reference \\
Attention + random & bf16 & 1e-4 & $0.670{\pm}0.009$ & $0.006{\pm}0.002$ & $0.905{\pm}0.014$ & $+0.001{\pm}0.002$ \\
Attention frozen & bf16 & 1e-4 & $0.666{\pm}0.010$ & $0.002{\pm}0.002$ & $0.939{\pm}0.009$ & not matched \\
LoRA & bf16 & 1e-4 & $0.624{\pm}0.008$ & $-0.040{\pm}0.012$ & $0.986{\pm}0.004$ & not matched \\
\bottomrule
\end{tabularx}
\par\smallskip
{\footnotesize Means and SEs use nine seeds. SR denotes stochastic rounding. Not matched means the paired acquisition difference exceeds $\max(0.031, 2\,\mathrm{SE})$. Variable acq.\ marks a row in which one seed acquired fewer than half of the facts. DEFT freezes 40\% of units. Score masks freeze 35.2\% of every feed-forward matrix with attention frozen.}